\documentclass[10pt,nonatbib, twocolumn]{interact}

\usepackage{epstopdf}
\usepackage[caption=false]{subfig}
\usepackage[numbers,sort&compress]{natbib}
\usepackage{hyperref}
\usepackage{style}
\usepackage{rotating}
\bibpunct[, ]{[}{]}{,}{n}{,}{,}
\renewcommand\bibfont{\fontsize{10}{12}\selectfont}

\theoremstyle{plain}
\theoremstyle{definition}
\theoremstyle{remark}

\usepackage[ruled]{algorithm}
\usepackage{algpseudocode}
\usepackage{enumitem}
\usepackage{fancyhdr}
\usepackage{url}
\usepackage{tabularx}
\usepackage{amssymb,amsfonts,amsmath}
\usepackage{pifont}
\usepackage{setspace}
\usepackage{multirow}
\usepackage{graphicx}
\usepackage{float}
\usepackage{soul}
\usepackage{epsfig}
\usepackage{textcomp}
\usepackage{longtable}
\usepackage{aligned-overset}
\usepackage{booktabs, makecell}
\usepackage{footnote}
\usepackage{amsmath}
\usepackage{subfig}

\newcommand{\cu}[1]{{\color{black}#1}}

\begin{document}

\title{\mtext{PersA-FL}: Personalized Asynchronous Federated Learning}

\author{
\name{Mohammad Taha Toghani\textsuperscript{a}\thanks{This work was partly done while MTT interning at Yahoo! Research. Part of this material is based upon work supported by the National Science Foundation under Grants \#2211815 and \#2213568. Corresponding Author's Email: mttoghani@rice.edu}, Soomin Lee\textsuperscript{b}, C{\'e}sar A. Uribe\textsuperscript{a}}
\affil{\textsuperscript{a}Department of Electrical and Computer Engineering, Rice University,
Houston, TX, USA \textsuperscript{b}Yahoo! Research, Sunnyvale, California, USA}
}

\maketitle

\begin{abstract}
We study the personalized federated learning problem under asynchronous updates. In this problem, each client seeks to obtain a personalized model that simultaneously outperforms local and global models. We consider two optimization-based frameworks for personalization: (i) Model-Agnostic Meta-Learning (\mtext{MAML}) and (ii) Moreau Envelope (\mtext{ME}). \mtext{MAML} involves learning a joint model adapted for each client through fine-tuning, whereas \mtext{ME} requires a bi-level optimization problem with implicit gradients to enforce personalization via regularized losses. We focus on improving the scalability of personalized federated learning by removing the synchronous communication assumption. Moreover, we extend the studied function class by removing boundedness assumptions on the gradient norm. Our main technical contribution is a unified proof for asynchronous federated learning with bounded staleness that we apply to \mtext{MAML} and \mtext{ME} personalization frameworks. For the smooth and non-convex functions class, we show the convergence of our method to a first-order stationary point. We illustrate the performance of our method and its tolerance to staleness through experiments for classification tasks over heterogeneous datasets.
\end{abstract}


\section{Introduction}\label{sec:introduction}

Federated Learning (\mtext{FL}) is designed to facilitate distributed training of machine learning models across devices by exploiting the data and computation power available to them \cite{konevcny2016federated}. A major benefit of \mtext{FL} is its ability to allow training models on data distributed across multiple devices without centralization. This is particularly beneficial in situations with limited sensitive data \cite{kairouz2019advances,kairouz2021distributed} where clients are reluctant to share their private data. At the same time, it is known that training over a larger set of data points improves the quality of the obtained model \cite{wang2021field}. In such scenarios, \mtext{FL} enjoys the power of collaborative learning without relocating the data from its original source \cite{kairouz2019advances}. Nevertheless, \mtext{FL} poses challenges such as data heterogeneity (statistical
diversity among clients) \cite{karimireddy2020scaffold,fallah2020personalized,dennis2020heterogeneity,li2020federated}, fairness \cite{divi2021new,li2021ditto}, privacy \cite{girgis2021shuffled,wei2020federated,huang2015differentially,prasad2022reconciling}, unreliable communication \cite{spiridonoff2020robust,toghani2022parspush}, and staleness \cite{xie2019asynchronous,assran2020advances,li2019asynchronous,niwa2021asynchronous}.

The common underlying assumption that determines the superiority of \mtext{FL} to individual local training is that the data points of all clients are coming from the same distribution, i.e., homogeneous data across clients. Consequently, \mtext{FL} can improve the quality of empirical loss minimization when data available on each device is limited; otherwise, each client may obtain a proper model without collaboration or communication with others. Therefore, \mtext{FL}\footnote{We refer to Federated Learning with no personalization as \mtext{FL}.} results in a common global model with better generalization across clients \cite{mansour2020three} compared to individual training. In heterogeneous data setups where clients hold samples from non-identical data distributions, a common (global) model may perform poorly on the local data points of each client. For instance, consider the next word prediction task on a smart keyboard \cite{hard2018federated}, where each client has a unique writing style or emphasis on the vocabulary domain. In this example, the corresponding mobile application is supposed to suggest a set of words that will likely be selected as the next word in the sentence. This scenario clearly states a case with a heterogeneous data setup with a limited sample on each client's device. Thus, if each client trains a model independently, without collaboration with the other clients, the model will likely perform poorly on the new data due to sample limitation. Hence, the question arises about what will occur if the clients hold data samples from similar (but not identical) distributions. 

In \mtext{FL} with heterogeneous data, an ideal scenario is to learn a globally common model easily adaptable to local data on each client, i.e., model fusion. This approach is known as \emph{Personalized Federated Learning} (\mtext{PFL}), which strives to exploit both the shared and unshared information from the data of all clients. A solution to the model fusion in \mtext{PFL} is to apply transfer learning \cite{zhuang2020comprehensive,dimitriadis2020federated} (e.g., fine-tuning) on a jointly trained model under \mtext{FL}. 
Interestingly, the centralized version of this problem has been extensively studied in Meta-Learning \cite{vanschoren2018meta} and Multi-Task Learning \cite{mortaheb2022fedgradnorm}, where the goal is to obtain a meta (global) model that with (potentially) minimal adaptation performs well on multiple tasks. Particularly, Model-Agnostic Meta-Learning (\mtext{MAML}) \cite{finn2017model,rajeswaran2019meta} proposes an optimization-based formulation that aims to find an initial meta-model with proper performance after applying one or a few steps of (stochastic) gradient descent. The key property of \mtext{MAML} is its ability to gauge fine-tuning during the learning process. Multiple studies have been conducted on the convergence and generalization of \mtext{MAML} \cite{fallah2020convergence,ji2020multi,finn2019online,fallah2021convergence,fallah2021generalization,charles2021convergence} for various problems and setups. \citet{fallah2020personalized} suggest the \mtext{MAML} formulation as a potential solution for \mtext{PFL}, and propose \mtext{Per-FedAvg} algorithm for collaborative learning with \mtext{MAML} personalized cost function. \citet{dinh2020personalized} present \mtext{pFedMe} algorithm for \mtext{PFL} via adopting a different formulation for personalization, namely Moreau Envelopes (\mtext{ME}). The proposed algorithm is a joint bi-level optimization problem with personalized parameters which are regularized to be close to the global model. We will elaborate on these two formulations (\mtext{MAML} \& \mtext{ME}) in Section \ref{sec:setup}. Additionally, several recent works have approached \mtext{PFL} mainly through optimization-based \cite{hanzely2020lower,lyu2022personalized,toghani2022parspush,hanzely2021personalized,huang2021personalized,mansour2020three,zhang2020personalized,farnia2022optimal,deng2020adaptive,bergou2022personalized,gasanov2021flix}, or structure-based \cite{collins2021exploiting,tziotis2022straggler,shamsian2021personalized} techniques.

Scalability to large-scale setups with potentially many clients is another major challenge for \mtext{FL}. The proposed algorithms in this scheme, mostly require synchronous communications between the server and clients \cite{mcmahan2017communication,konevcny2016federated,gasanov2021flix,deng2020adaptive,li2021ditto,fallah2020personalized,dinh2020personalized}. Such constraints impose considerable delays on the learning progress, since increasing the concurrency in synchronous updates decreases the training speed and quality. For example, limited communication bandwidth, computation power, and communication failures incur large delays in the training process. In cross-device \mtext{FL}, devices are naturally prone to update and communicate models under less restrictive rules, whereas clients may apply updates in an asynchronous fashion, i.e., staleness. Hogwild! \cite{niu2011hogwild} is one of the first efforts to model asynchrony in distributed setup with delayed updates. Multiple works have studied asynchronous training under different setups and assumptions \cite{agarwal2012distributed,mitliagkas2016asynchrony,bertsekas2021distributed,niwa2021asynchronous,li2019asynchronous,aviv2021learning,eichner2019semi}. Specifically, some recent seminal works have studied the convergence of asynchronous \mtext{SGD}-based methods, and show their convergence under certain assumptions on maximum or average delay \cite{arjevani2020tight,stich2021critical,koloskova2022sharper,mishchenko2022asynchronous}.\footnote{\citet{mishchenko2022asynchronous} studies the convergence of distributed optimization for homogeneous strongly convex and smooth functions with no assumptions on maximum delay, i.e., unbounded staleness.} In decentralized setups, \citet{hadjicostis2015robust} propose a consensus algorithm called running-sum, which is robust to message losses. Furthermore, \citet{olshevsky2018fully} present a more general framework with robustness to asynchrony, message losses, and delays for both consensus and optimization problems \cite{spiridonoff2020robust,toghani2022parspush}. More closely, \mtext{FL} under stale updates has been thoroughly studied in \cite{xie2019asynchronous,nguyen2022federated,assran2020advances,reisizadeh2022straggler,kulkarni2020survey,tziotis2022straggler}. Particularly, \citet{tziotis2022straggler} studies the existence of stragglers in \mtext{PFL} via shared representations, i.e., system and data heterogeneity in structure-based personalization.

\cu{The main contribution of paper \cite{nguyen2022federated} is on the server algorithm, where this paper proposes a more secure and robust algorithm by aggregating a buffer of asynchronous updates within a secure channel prior to sending them to the server. Whereas, our work focuses on scalability and personalization via asynchronous communication and learning personalized models.}

In this work, we study the \mtext{PFL} problem under asynchronous communications to improve training concurrency, performance, and efficiency. We propose the \mtext{PersA-FL} algorithm, a novel personalized \& asynchronous method that jointly addresses the heterogeneity and staleness in \mtext{FL}. We develop a technique based on asynchronous updates to resolve the communication bottleneck imposed by synchronized learning in \mtext{PFL}, where we improve the training scalability and performance. To the best of our knowledge, this is the first study on the intersection of staleness and personalization through the lens of optimization-based techniques. We summarize our contributions as follows:
\begin{itemize}
    \item Through the integration of two personalization formulations, \mtext{MAML} \& \mtext{ME}, we propose \mtext{PersA-FL}, an algorithm that allows personalized federated learning under asynchronous communications between the server and clients. Our proposed method consists of two algorithms from the perspectives of the server and clients. We present the client algorithm under three different options for the local updates, each addressing a separate formulation, \textcolor{royalblue}{(A) \mtext{FedAsync}}, \textcolor{brickred}{(B) \mtext{PersA-FL-MAML}}, and \textcolor{seagreen}{(C) \mtext{PersA-FL-ME}}.
    \item \cu{We present a new convergence analysis for Asynchronous Federated Learning (\mtext{FedAsync}) under smooth non-convex cost functions by removing the boundedness assumption from the gradient norm.} Our analysis assumes bounded variance of stochasticity and heterogeneity, and bounded maximum delay. Hence, we improve the existing theory by extending the result to a broader function class, i.e., unbounded gradient norm.
    \item We show the convergence rate of \textcolor{brickred}{\mtext{PersA-FL-MAML}} based on the maximum delay and personalization budget under the same assumptions as \citet{fallah2020personalized}.\footnote{Besides the assumptions for \mtext{FedAsync}, seminal works \cite{fallah2020personalized,finn2019online,rajeswaran2019meta} assume second-order Lipschitzness, bounded variance, and bounded gradient in the analysis of \mtext{MAML} cost functions.} We highlight the impact of batch size in the biased stochastic estimation of the full gradients for the \mtext{MAML} cost. \cu{We present the communication and sample complexity to find an $\varepsilon$ first-order stationary point for the proposed algorithm.}
    \item We prove the convergence of \textcolor{seagreen}{\mtext{PersA-FL-ME}} with no boundedness assumption on the gradient norm. We discuss the connection of convergence rate to the gradient estimation error and level of personalization. Compared to \cite{dinh2020personalized}, we show an explicit dependence of convergence rate to the estimation error. Moreover, we relax the heterogeneity assumption in \cite{dinh2020personalized} allowing bounded population diversity instead of uniformly bounded heterogeneity. \cu{We determine the communication and local inexact solver complexity to find an $\varepsilon$ first-order stationary point for this method.}
    \item We present numerical experiments evaluating our proposed algorithm on heterogeneous MNIST and CIFAR10 with unbalanced distributions across the clients. We illustrate the advantages of our method in terms of performance and scalability to varying delays in setups with heterogeneity.
\end{itemize}

Table \ref{tab:comparison} illustrates the properties of our proposed method and provides a comparison between our algorithm and underlying analysis with related seminal works. \cu{As shown in this table, building upon the results in \cite{fallah2020personalized,dinh2020personalized}, we extend the capability of \mtext{FL} to staleness.} Table \ref{tab:comparison} also contains the convergence results for our proposed algorithm, which we will discuss in more details in Section \ref{sec:convergence}. 

\cu{The main difference between our method and the works in \cite{xie2019asynchronous,nguyen2022federated} mainly lies in the client algorithm, where we consider three options (\textcolor{royalblue}{A}, \textcolor{brickred}{B}, and \textcolor{seagreen}{C}) for updating the parameters locally. \textcolor{royalblue}{Option A} is similar to the client algorithm in \cite{xie2019asynchronous,nguyen2022federated}, but we improve the theoretical convergence results by removing the assumption on bounded gradients for this setup. \textcolor{brickred}{Option B} and \textcolor{seagreen}{Option C}, along with the server algorithm are novel methods for personalized asynchronous federated learning.  \citet{nguyen2022federated} characterize the server algorithm with a secure and robust update aggregation and \cite{toghani2022unbounded} enhances its theoretical properties. Study of secure aggregation on the server side remains as a future direction for this work.}

\begin{table}[t]
\caption{A comparison of related federated learning methods with convergence guarantees for smooth non-convex functions. Parameters $\tau$, $\alpha$, $\nu$, and $b$ respectively denote the maximum delay, \mtext{MAML} personalization stepsize, \mtext{ME} inexact gradient estimation error, batch size.}\label{tab:comparison}
\resizebox{\textwidth}{!}{
\centering
\begin{tabular}{llcccl}
\toprule
\bf Algorithm & \bf \&\,\,\, Reference&
\bf \rotatebox[origin=c]{90}{Personalized} \rotatebox[origin=c]{90}{Cost} &
\bf \rotatebox[origin=c]{90}{Asynchronous} \rotatebox[origin=c]{90}{Updates} &
\bf \rotatebox[origin=c]{90}{Unbounded} \rotatebox[origin=c]{90}{Gradient} & 
\bf \hspace{2.3em} Convergence Rate
\\
\midrule
\hspace{-0.7em}
& \citet{mcmahan2017communication} & \xmark & \xmark & - & No Analysis\\
\cmidrule(r){2-6}
\mtext{FedAvg} & \citet{yu2019parallel} & \xmark & \xmark & \xmark & $\mcO\left(\frac{1}{\sqrt{T}}\right)$\\
\cmidrule(r){2-6}
& \citet{wang2020tackling} & \xmark & \xmark & \cmark & $\mcO\left(\frac{1}{\sqrt{T}}\right)$\\
\midrule
\multirow{3}{*}{\colorbox{royalblue!20}{\mtext{\cu{FedAsync}}}}
& \citet{xie2019asynchronous} & \xmark & \cmark & \xmark & $\mcO\left(\frac{1}{\sqrt{T}}\right) + \mcO\left(\frac{\tau^2}{T}\right)$\\
\cmidrule(r){2-6}
\hspace{-1.3em}
& \textcolor{magenta}{This Work} & \xmark & \cmark & \cmark & \cu{$\mcO\left(\frac{1}{\sqrt{T}}\right) + \mcO\left(\frac{\tau^2}{T}\right)$}\\
\midrule
\multirow{3}{*}{\mtext{FedBuff}} & \citet{nguyen2022federated} & \xmark & \cmark & \xmark & $\mcO\left(\frac{1}{\sqrt{T}}\right) + \mcO\left(\frac{\tau^2}{T}\right)$\\
\cmidrule(r){2-6}
\hspace{-1.3em}
& \cu{\citet{toghani2022unbounded}} & \xmark & \cmark & \cmark & \cu{$\mcO\left(\frac{1}{\sqrt{T}}\right) + \mcO\left(\frac{\tau^2}{T}\right)$}\\
\midrule
\mtext{Per-FedAvg} & \citet{fallah2020personalized} &
\begin{tabular}{c}
\cmark
\end{tabular}
& \xmark & \xmark & $\mcO\left(\frac{1}{\sqrt{T}}\right) + \mcO\left(\frac{\alpha^2}{b}\right)$\\
\midrule
\mtext{pFedMe} & \citet{dinh2020personalized} &
\begin{tabular}{c}
\cmark
\end{tabular}& \xmark & \cmark & $\mcO\left(\frac{1}{\sqrt{T}}\right)
+\mcO\left(\frac{\lambda^2\left(\frac{1}{b}+\nu^2\right)}{(\lambda{-}L)^2}\right)$\\
\midrule
\colorbox{brickred!20}{\mtext{PersA-FL-MAML}} & \textcolor{magenta}{This Work} &
\begin{tabular}{c}
\cmark
\end{tabular}
& \cmark & \xmark & $\mcO\left(\frac{1}{\sqrt{T}}\right) + \mcO\left(\frac{\tau^2}{T}\right) + \mcO\left(\frac{\alpha^2}{b}\right)$\\
\midrule
\colorbox{seagreen!20}{\mtext{PersA-FL-ME}}& \textcolor{magenta}{This Work} &
\begin{tabular}{c}
\cmark
\end{tabular}
& \cmark & \cmark & $\mcO\left(\frac{1}{\sqrt{T}}\right) + \mcO\left(\frac{\tau^2}{T}\right) + \mcO\left(\frac{\lambda^2}{(\lambda{-}L)^2}\nu^2\right)$\\
	\bottomrule      
	\end{tabular}     
	}
\end{table}

The remainder of this paper is organized as follows. In Section \ref{sec:setup}, we introduce the \mtext{PFL} setup and discuss the asynchronous communication framework between the server and clients. In Section \ref{sec:algorithm}, we describe our algorithm, PersA-FL, for \mtext{PFL} under staleness. In Section \ref{sec:convergence}, we state the convergence result for our proposed algorithm along with the underlying assumptions and technical lemmas. We present the numerical experiments in Section \ref{sec:experiments}. We finally end by concluding remarks in Section \ref{sec:conclusion}.

\section{Problem Setup \& Background}\label{sec:setup}

In this section, we first present the formal problem setup for \mtext{FL} \cite{mcmahan2017communication}, as well as the personalization formulations in \mtext{MAML} \cite{fallah2020personalized} and \mtext{ME} \cite{dinh2020personalized}. Then, we discuss the underlying communication setting under asynchronous updates.

\subsection{Federated Learning Problem Setup}\label{subsec:fl}

We consider a set of $n$ clients and one server, where each client $i\in[n]$ holds a private function $f_i:\bbR^d \to \bbR$, and the goal is to collaboratively obtain a model $w\in\bbR^d$ that minimizes the local cost functions on average, as follows:
\begin{align}\label{eq:fl}
\begin{split}
    \min_{w\in\bbR^d} f(w)&\coloneqq\frac{1}{n}\sum\limits_{i=1}^{n}f_i(w),\\
    \text{with}\quad f_i(w) &\coloneqq \bbE_{\Xi_i\sim p_i} [\ell_i(w,\Xi_i)],
\end{split}
\end{align}
where $\ell_i:\bbR^d\times \mcS_i \to \bbR$ is a cost function that determines the prediction error of some model $w\in\bbR^d$ over a single data point $\xi_i\in\mcS_i$ on client $i$, where $\xi_i$ is a realization of $\Xi_i \sim p_i$, i.e., $p_i$ is the client $i$'s data distribution over $\mcS_i$, for $i\in[n]$. In the above definition, $f_i(\cdot)$ is the local cost function of client $i$, and $f(\cdot)$ denotes the global cost function, i.e., average loss. For instance, in a supervised learning setup with $\mcZ_i\coloneqq \mcX_i\times\mcY_i$, we have $\ell_i(w,\xi_i)$ as the prediction cost of some learning model parameterized by $w$ for sample $\xi_i = (x,y)$, where $x\in\mcX_i$ and $y\in\mcY_i$. Let $\mcD_i$ be a data batch with samples independently drawn from the distribution $p_i$. Then, the unbiased stochastic cost associated with data batch $\mcD_i$ can be denoted as follows:
\begin{align}\label{eq:stoch-loss}
\tilde{f}_i(w,\mcD_i) &\coloneqq \frac{1}{|\mcD_i|} \sum\limits_{\xi_i\in\mcD_i}\ell_i(w,\xi_i),
\end{align}
where for simplicity, we assume that the size of all batches is larger than $b$. Then, according to the above definition, we can immediately infer that
\begin{align}\label{eq:unbiased-stoch}
\begin{split}
\bbE_{p_i}\left[\tilde{f}_i(w,\mcD_i)\right] &= f_i(w),\\ \bbE_{p_i}\left[\nabla\tilde{f}_i(w,\mcD_i)\right] &= \nabla f_i(w),\\ \bbE_{p_i}\left[\nabla^2\tilde{f}_i(w,\mcD_i)\right] &= \nabla^2 f_i(w).
\end{split}
\end{align}

Several works have been proposed to solve \eqref{eq:fl} as a union of local and global optimization steps. For instance, \mtext{FedAvg} \cite{mcmahan2017communication} suggests an iterative algorithm wherein at each round $t\geq 0$, (i) server transmits its current parameter $w^{t}$ to a subset of the clients, (ii) each selected client updates the parameter locally, by applying $Q$ sequential rounds of stochastic gradient descent (\mtext{SGD}) with respect to its local cost function, then (iii) the selected clients send back their local parameter to the server, and finally, (iv) the server aggregates the so-called local parameters to obtain a new global parameter $w^{t+1}$. As a result, clients minimize the average loss in \eqref{eq:fl} with less communication cost, i.e., fewer global rounds. Note that the underlying assumption for methods such as \mtext{FedAvg} is the possibility of synchronized communications between the selected clients and the server. The left chart in Figure \ref{fig:schedule} represents the communication and update schedule for \mtext{FedAvg}. The performance of \mtext{FL}-based methods depends on the similarity of distributions $\mcD_i$, thus, cases with heterogeneous datasets slow down the convergence. \citet{karimireddy2020scaffold} and \cite{dennis2020heterogeneity} the effect of heterogeneity in the convergence speed. A solution of \eqref{eq:fl} is a common model for all the clients; hence no adaptation or fusion to each client's data. Next, we elaborate on the personalization concept and discuss two alternative problem formulations for \eqref{eq:fl}.

\subsection{Personalized Federated Learning}

In the previous section, we explained how a solution to \eqref{eq:fl} performs well when the data is homogeneous, and the goal is to obtain a shared model. On the one hand, using a single common model, with no adaptation to each client, does not necessarily lead to a proper performance when dealing with heterogeneous datasets. On the other hand, when the data distributions of different clients share some similarities, e.g., bounded variance in their heterogeneity, and the number of data points on each client is limited, joint training with fusion improves the performance compared to individual locally trained models or \mtext{FL}. Therefore, learning a shared model with little fine-tuning, e.g., a few steps of \mtext{SGD} with respect to the local cost, may result in a proper personalized model.

\citet{fallah2020personalized} proposed \mtext{Per-FedAvg} algorithm, which modifies the training loss function by taking advantage of the fact that fine-tuning will occur after training. The \mtext{MAML} formulation assumes a limited computational budget for personalization (fine-tuning) at each client. It then offers to look for an initial (global) parameter that performs well after it is updated with one or a few steps of \mtext{SGD}. In other words, \cite{fallah2020personalized} define the \mtext{MAML} loss function for \mtext{PFL} as follows:
\begin{align}\label{eq:persafl-maml}
\begin{split}
\min_{w \in \bbR^d} F^{(b)}(w) &\coloneqq \frac{1}{n} \sum\limits_{i=1}^{n} F^{(b)}_i(w),\\
\text{with}\quad F^{(b)}_i(w) &\coloneqq f_i(w - \alpha \nabla f_i(w)),
\end{split}
\end{align}
where $\alpha\geq 0$ is the \mtext{MAML} personalization stepsize. Solving \eqref{eq:persafl-maml} yields a global (meta) model that can be used to create a personalized model by applying one step of gradient descent with respect to individual loss functions. The degree of fine-tuning determines the personalization budget, which often controls the trade-off between having a local (personalized) or generic model, i.e., exploiting the shared and local knowledge simultaneously. In Problem \eqref{eq:persafl-maml}, stepsize $\alpha$ determines the personalization budget, where $\alpha=0$ implies \mtext{FL} in Problem \eqref{eq:fl}. See \cite{ji2020multi,toghani2022parspush,fallah2021convergence} for the study of multi-step \mtext{MAML}. In a nutshell, \mtext{Per-FedAvg} proposes to minimize $F^{(b)}(w)$ via a similar paradigm as \mtext{FedAvg}. Hence, each client $i$ computes the personalized gradient of its \mtext{MAML} cost in \eqref{eq:persafl-maml}, which can be written as follows:
\begin{align}\label{eq:persafl-maml-full-grad}
    \nabla F^{(b)}_i(w) = \left[I {-} \alpha \nabla^2 f_i(w)\right]\nabla f_i\left(w{-}\alpha \nabla f_i(w)\right),
\end{align}
where in \mtext{Per-FedAvg}, the authors propose to compute a biased estimation of \eqref{eq:persafl-maml-full-grad} using stochastic gradients/Hessian. We will elaborate on the stochastic approximation in Section \ref{sec:algorithm}.

On a separate note, one of the major challenges in \mtext{Per-FedAvg} is the computation of second-order information such as Hessian for large-scale models (large $d$). However, as proposed by \cite{fallah2020personalized}, one can skip the Hessian in the gradient formulation (\mtext{FO-MAML}) or approximate it with first-order information (\mtext{HF-MAML}) \cite{fallah2020convergence}.

As an alternative option to \mtext{MAML} formulation in \eqref{eq:persafl-maml}, \citet{dinh2020personalized} suggest solving the following optimization problem:
\begin{align}\label{eq:persafl-me}
\begin{split}
\min_{w \in \bbR^d} F^{(c)}(w) &\coloneqq \frac{1}{n} \sum\limits_{i=1}^{n} F^{(c)}_i(w),\\
\text{with}\quad F^{(c)}_i(w) &\coloneqq \min_{\theta_i \in \bbR^d} \left[f_i(\theta_i) + \frac{\lambda}{2} \norm{\theta_i - w}^2 \right],
\end{split}
\end{align}
where each function $F^{(c)}_i(w)$ is a local cost of personalized parameter $\theta_i\in\bbR^d$ by using the Moreau Envelope as a regularized loss function, and parameter $\lambda\geq 0$ determines the degree of personalization. In this setup, $\lambda=0$ is equivalent to local training with no collaboration and as $\lambda\to\infty$, the formulation in \eqref{eq:persafl-me} converges to \mtext{FL} in \eqref{eq:fl} with no personalization which is similar to the case in \eqref{eq:persafl-maml} with $\alpha=0$. For non-extreme values of $\lambda$, the clients jointly learn a global model $w$ and personalized parameters $\theta_i$, which are regularized to remain close to $w$. Note that the gradient of $F^{(c)}_i(w)$ can be written as follows (please check out Appendix \ref{app:persafl-me} to see the proof):
\begin{align}
    \nabla F^{(c)}_i(w) &= \lambda\left(w-\hat{\theta}_i(w)\right),\label{eq:persafl-me-full-grad}\\
    \text{with}\quad \hat{\theta}_i(w)\coloneqq &\argmin_{\theta_i \in \bbR^d} \left[f_i(\theta_i) + \frac{\lambda}{2} \norm{\theta_i - w}^2 \right],\label{eq:persafl-me-argmin}
\end{align}
where for large $\lambda$, $\hat\theta_i(w)$ is the exact solution to an optimization problem. Therefore, solving \eqref{eq:persafl-me} through a similar approach to \mtext{FedAvg} or \mtext{Per-FedAvg}, itself requires minimizing Problem \eqref{eq:persafl-me-argmin} which is potentially intractable. \citet{dinh2020personalized} propose a bi-level optimization algorithm called \mtext{pFedMe}, to minimize the optimization problem in \eqref{eq:persafl-me} by alternating minimization over $\theta_i$ and $w$. The main idea behind \mtext{pFedMe} is to integrating the computation of an inexact solution to \eqref{eq:persafl-me-argmin} inside an \mtext{FL}-type method. We will explain and use this inexact approximation in the presentation of our method (\textcolor{seagreen}{Option C}) in Section \ref{sec:algorithm}.

\subsection{Asynchronous vs Synchronous Schedule}\label{subsec:asynch}

So far, we have discussed the three different formulations for collaborative learning that we will consider in our method. As we described the \mtext{FedAvg} algorithm in Subsection \ref{subsec:fl}, at each round $t$, the parameter $w^t$, which is the most recent version of the global parameter in the server, will be sent to a subset of the clients. Then, the server halts the training process until all selected clients receive this parameter, perform local updates, and transmit their updates back to the server. This synchronization procedure restricts the algorithm flow to the slowest client at each round. Nevertheless, asynchronous updates and communications can be described in this described framework.

Let us provide a comparison using the example in Figure \ref{fig:schedule} which illustrates the communication and update schedule for synchronous (left) \& asynchronous (right) aggregations for $n=5$ clients in \mtext{FL} with $Q=3$ local updates. As shown in this Figure, for every update at the server-lever under synchronized updates (left figure), the server has to wait for all the selected clients. Nevertheless, these clients build their local updates based on the recent version of the server's parameter. On the contrary, in the asynchronous scenario (right figure), the server updates the global parameter once it receives a new update from some client. The main challenge for the asynchronous setup is the staleness between download and upload time from/to the server. We design \mtext{PersA-FL} based on the second communication scenario.

\begin{figure}[t]
    \centering
    \includegraphics[width=0.49\linewidth]{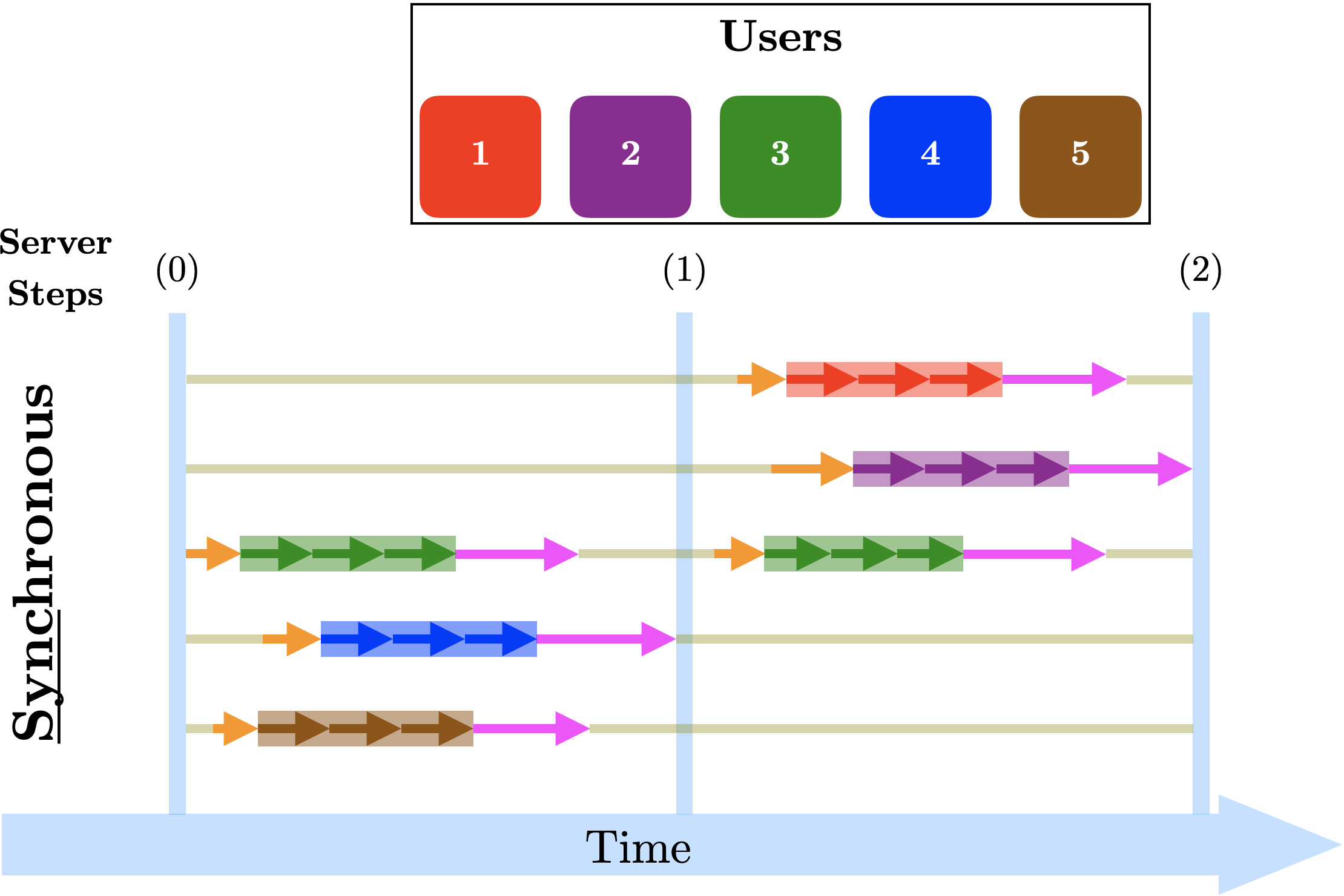}
    \includegraphics[width=0.49\linewidth]{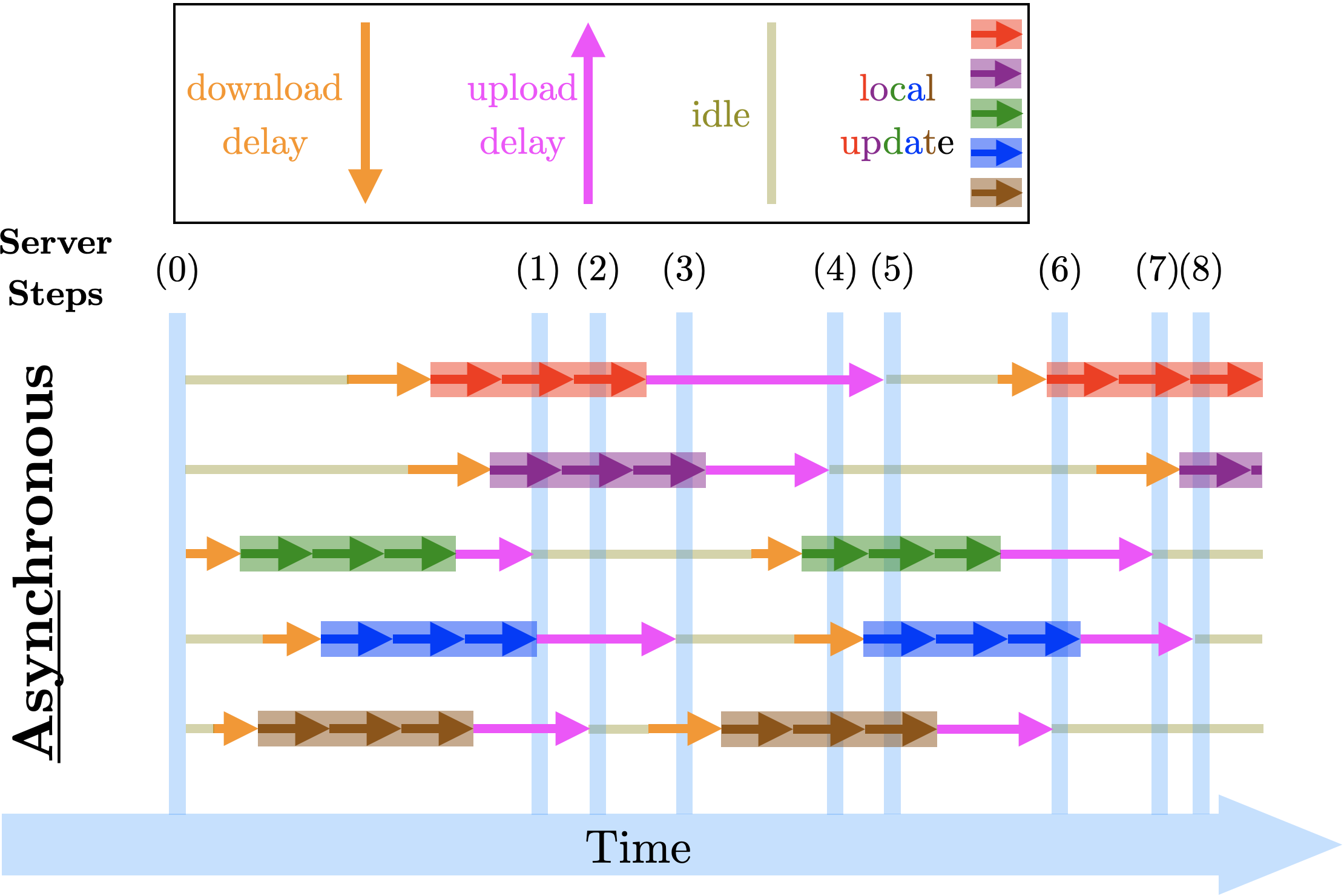}
    \caption{Communication and update schedule for synchronous and asynchronous aggregation: The demonstrated setup in this example contains $n=5$ clients with $Q=3$ local updates.}
    \label{fig:schedule}
\end{figure}

\section{Algorithm: \mtext{FedAsync} \& \mtext{PersA-FL} }\label{sec:algorithm}

In this section, by integrating the problem formulations in \eqref{eq:fl}, \eqref{eq:persafl-maml}, and \eqref{eq:persafl-me} into a united format, we propose Algorithms \ref{alg:server} \& \ref{alg:client} to solve these problems under three different update choices at the client-level. We present our method through two different perspectives, (i) server and (ii) client.

\noindent\textbf{$\diamond$ Server Algorithm:} Let us denote $w^0\in\bbR^d$ as the initial parameter at the server, where the objective is to minimize the cost function in either \eqref{eq:fl}, \eqref{eq:persafl-maml}, or \eqref{eq:persafl-me}. Each client $i\in[n]$ may communicate with the server when the underlying connection is stable. Clients may request to download the server's parameters at any time, and the server will send the most recent model after receiving the request. All underlying delays for the communications between the server and clients are modeled as download and upload delays. We consider variable $t$ as a counter for the updates at the server level. Algorithm \ref{alg:server} represents the server updates in \mtext{PersA-FL}. The server performs an iterative algorithm where at each round $t\geq 0$, remains on hold until receives an update $\Delta_{i_t}\in\bbR^d$ from some client $i_t\in[n]$. After receiving the update from client $i_t$, the server updates its parameter according to Step \ref{ln:server-update} of Algorithm \ref{alg:server}, where $\beta \geq 0$ is the server stepsize.

\begin{algorithm}[!ht]
    \caption{[Personalized] Asynchronous Federated Learning (\textbf{Server})}
    \begin{algorithmic}[1]
    \State{\textbf{input:} model $w^{0}$, $t=0$, server stepsize $\beta$.}
    \Repeat
        \If{the server receives an update $\Delta_{i_t}$ from some client $i_t{\in}[n]$}
            \State{$w^{t+1} \gets w^{t} - \beta \Delta_{i_t}$}\label{ln:server-update}
            \State{$t \gets t+1$}
        \EndIf 
    \Until{not converge}
    \end{algorithmic}
    \label{alg:server}
\end{algorithm}

\begin{algorithm}[!ht]
    \caption{[Personalized] Asynchronous Federated Learning (\textbf{Client} $i$)}
    \begin{algorithmic}[1]
    \State{\textbf{input:} number of local steps $Q$, local stepsize $\eta$, \mtext{MAML} stepsize $\alpha$, \mtext{ME} regularization parameter $\lambda$, minimum batch size $b$, estimation error $\nu$.}
    \Repeat
    \State{read $w$ from the server} \Comment{download phase}\label{ln:client-download}
    \State{$w_{i,0} \gets w$}
    \For{$q=0$ to $Q{-}1$}\label{ln:client-for} \Comment{local updates}
    \State{sample a data batch $\mcD_{i,q}$ from distribution $p_i$ \hfill $\triangledown$ 3 options:}
    
    \vspace{0.05em}

    \hspace{-2.28em}\colorbox{royalblue!20}{\parbox{\linewidth}{
    
    \Statex{\textcolor{royalblue}{\quad\qquad $\triangleright$ \textbf{Option A} (FedAsync)}}\label{opt:A}
    \vspace{0.2em}
    \State{
    $w_{i,q+1} \gets w_{i,q} - \eta \nabla \tilde{f}_i(w_{i,q},\mcD_{i,q})$
    }
    }}
    
    \vspace{0.05em}

    \hspace{-2.28em}\colorbox{brickred!20}{\parbox{\linewidth}{
    
    \Statex{\quad\qquad $\triangleright$ \textcolor{brickred}{\textbf{Option B} (\mtext{PersA-FL-MAML})}}\label{opt:B}
    \vspace{0.2em}
    \State{ sample two data batches $\mcD_{i,q}',\mcD_{i,q}''$ from distribution $p_i$}
    
    \State{ $w_{i,q+1} \gets w_{i,q} - \eta \left[I {-} \alpha \nabla^2 \tilde{f}_i(w_{i,q},\mcD_{i,q}'')\right]\nabla\tilde{f}_i\left(w_{i,q}{-}\alpha \nabla \tilde{f}_i(w_{i,q},\mcD_{i,q}'),\mcD_{i,q}\right)$}
    
    }}

    \vspace{0.05em}
    
    \hspace{-2.28em}\colorbox{seagreen!20}{\parbox{\linewidth}{
    
    \Statex{\quad\qquad $\triangleright$ \textcolor{seagreen}{\textbf{Option C} (\mtext{PersA-FL-ME})}}\label{opt:C}
    \vspace{0.2em}
    \State{ $\tilde{h}_{i}(\theta_i,w_{i,q},\mcD_{i,q})\coloneqq \tilde{f}_i(\theta_i,\mcD_{i,q}) + \frac{\lambda}{2} \left\lVert \theta_i - w_{i,q} \right\rVert^2$}
    \State{ minimize $\tilde{h}_{i}(\theta_i, w_{i,q},\mcD_{i,q})$ w.r.t. $\theta_i$ up to accuracy level $\nu$ to find $\tilde{\theta}_{i}(w_{i,q})$:
    
    \begin{center}
    $\left\lVert\nabla\tilde{h}_i\left(\tilde{\theta}_{i}(w_{i,q}),w_{i,q},\mcD_{i,q}\right)\right\rVert\leq \nu$
    \end{center}
    }\label{ln:client-me-nu}
    \State{ $w_{i,q+1} \gets w_{i,q} - \eta\lambda (w_{i,q} - \tilde{\theta}_{i}(w_{i,q}))$}
    
    }}

    \vspace{0.3em}
    \EndFor\label{ln:client-endfor}
    \State{$\Delta_i \gets w_{i,0} - w_{i,Q}$}
    \State{client $i$ broadcasts $\Delta_{i}$ to the server}\Comment{upload phase}\label{ln:client-upload}
    \Until{not interrupted by the server}
    \end{algorithmic}
    \label{alg:client}
\end{algorithm}

Now, we are ready to present the client algorithm. Before starting, note that we drop the time index from the iterates of the client algorithm for clarity of exposition.

\noindent \textbf{$\diamond$ Client Algorithm}: Let us explain the operations of $i$-th client using the pseudo code in Algorithm \ref{alg:client}. Client $i$ repeats an iterative procedure which is composed of three phases, (i) downloading the most up-to-date model from the server as in Step \ref{ln:client-download}, (ii) performing $Q$ local updates starting from the parameters of the downloaded model with respect to the cost function of the underlying problem, \eqref{eq:fl}, \eqref{eq:persafl-maml}, or \eqref{eq:persafl-me}, as in Steps \ref{ln:client-for}-\ref{ln:client-endfor}, and (iii) uploading the sum of updates on the server as in Step \ref{ln:client-upload}. Note that $\eta \geq 0$ is the local stepsize, a hyperparameter. The main idea for the local updates is to perform $Q$ sequential \mtext{SGD} steps on the local cost. Below, we list our stochastic estimation for the full gradients of each loss function introduced in Section \ref{sec:setup}:
\begin{itemize}[leftmargin=7mm]
\item \textbf{\textcolor{royalblue}{Option A:}} This option intends to minimize \eqref{eq:fl}. Therefore, for each client $i$ at each local round $q$, we sample an independent data batch from $p_i$ and compute an unbiased estimation of the loss as in \eqref{eq:stoch-loss}.
\item \textbf{\textcolor{brickred}{Option B:}} By performing this option, we aim to minimize the \mtext{MAML} cost function in \eqref{eq:persafl-maml}. As we saw in Section \ref{sec:setup}, the full gradient can be computed according to \eqref{eq:persafl-maml-full-grad}. Following \cite{fallah2020personalized}, we sample three data batches to compute a biased estimation of \eqref{eq:persafl-maml-full-grad} as follows:
\begin{align}\label{eq:persafl-maml-stoch-grad}
    \nabla \tilde{F}^{(b)}_i(w,\mcD_i'',\mcD_i',\mcD_i) = \left[I {-} \alpha \nabla^2 \tilde{f}_i(w,\mcD_{i}'')\right]\nabla\tilde{f}_i\left(w{-}\alpha \nabla \tilde{f}_i(w,\mcD_{i}'),\mcD_{i}\right).
\end{align}

We will discuss the variance and bias of this estimator in Subsection \ref{subsec:persafl-maml}
\item \textbf{\textcolor{seagreen}{Option C:}} Finally, we invoke this option to minimize the \mtext{ME} personalized loss in \eqref{eq:persafl-me}. As we mentioned earlier, the full gradient of this cost is \eqref{eq:persafl-me-full-grad}, where for a fixed $w$, we may obtain $\hat{\theta}_i(w)$ by minimizing \eqref{eq:persafl-me-argmin}. Instead, following \cite{dinh2020personalized}, we define the stochastic approximation $\tilde{h}_i(\theta_i,w,\mcD_i)$ as in Step \ref{ln:client-me-nu}, and minimize this function with respect to $\theta_i$ to obtain an approximate solution $\tilde{\theta}_i(w)$ where the gradient's norm is less than some threshold $\nu\geq 0$. Therefore, we approximate \eqref{eq:persafl-me-full-grad} with the following estimator:
\begin{align}\label{eq:persafl-me-stoch-grad}
    \nabla \tilde{F}^{(c)}_i(w, \mcD_i) = \lambda\left(w-\tilde{\theta}_i(w)\right).
\end{align}
Let us denote the expectation of $\tilde{h}_i(.)$ as $h_i(.)$. Then, for $\lambda > L$, the expected function is $(\lambda{+}L)$-smooth and $(\lambda{-}L)$-strongly convex due to the properties of Moreau Envelopes \cite{dinh2020personalized}. Then according to the property of  \cite{bubeck2015convex,dinh2020personalized}, for some $\nu \leq 1$ (e.g., $10^{{-}5}$), we can find $\tilde{\theta}_i(w)$ in $\mcO(\frac{\lambda{+}L}{\lambda{-}L}\log (\frac{1}{\nu}))$ iterations.

We will also discuss the properties of \eqref{eq:persafl-me-stoch-grad} in Subsection \ref{subsec:persafl-me}.
\end{itemize}

Next, we present the convergence result of our method for the three formulations.

\section{Convergence Results}\label{sec:convergence}

In this section, we introduce the technical theorems and lemmas to show the convergence of our method for the three described scenarios. First, we introduce the common assumptions we will use in our analysis for all the three choices of Algorithm \ref{alg:client}. As mentioned earlier, we require some additional assumptions to show the convergence of \mtext{MAML}, which we will introduce in Subsection \ref{subsec:persafl-maml}. After stating the assumptions, we will present the convergence results.

Recall that the server updates its model at round $t$ using the updates sent by client $i_t\in[n]$. We denote $\Omega(t)$ as the timestep of the round at which client $i_t$ has received the server's parameters before applying its $Q$ local updates. In other words, $(\Omega(t), t)$ denote the download and upload rounds for client $i_t$. Now, we introduce the assumption of maximum delay.

\begin{assumption}[Bounded Staleness]\label{assump:staleness}
For all server steps $t\geq 0$, the staleness or effective delay between the model version at the download step $\Omega(t)$ and upload step $t$ is bounded by some constant $\tau$, i.e.,
\begin{align}\label{eq:staleness}
    \sup_{t\geq 0} \left|t-\Omega(t)\right|\leq \tau,
\end{align}
and the server receives updates uniformly, i.e., $i_{t} \sim \mathrm{Uniform}([n])$.
\end{assumption}
The above assumption is standard in the analysis of asynchronous methods, specifically in heterogeneous settings \cite{nguyen2022federated,xie2019asynchronous,assran2020advances,koloskova2022sharper,stich2021critical,arjevani2020tight}. Assumption \ref{assump:staleness} guarantees that all clients remain active over the course of training. However, they have transient delays and perform updates with staleness.

Next, we present our only assumption on the function class, i.e., smooth non-convex.


\begin{assumption}[Smoothness]\label{assump:smoothness}
For all clients $i\in[n]$, function $f_i:\bbR^d\to\bbR$ is bounded below, differentiable, and $L$-smooth, i.e., for all $w,u\in \bbR^d$,
\begin{align}
        \left\lVert\nabla f_i(w)-\nabla f_i(u)\right\rVert\leq L\lVert w-u\rVert\label{eq:smoothness}\\
        f_i^\star\coloneqq\min_{w\in\bbR^d} f_i(w) > -\infty.\label{eq:lower-bound}
\end{align}
\end{assumption}
The smoothness assumption is conventional in the analysis of non-convex functions. We also assume boundedness from below, which is reasonable since the ultimate goal is to minimize the functions.
We also denote $f^\star = \min_{i\in[n]}f_i^\star$, where according to this definition, we can immediately see that $f^\star \leq \min_{w\in\bbR^d}F^{(b)}(w)$ and $f^\star \leq \min_{w\in\bbR^d}F^{(c)}(w)$.

Now, we present our assumptions on bounded stochasticity and heterogeneity. 

\begin{assumption}[Bounded Variance]\label{assump:bounded-variance}
For all clients $i\in[n]$, the variance of a stochastic gradient $\nabla \ell_i(w,\xi_i)$ on a single data point $\xi_i\in\mcS_i$ is bounded, i.e., for all $w\in\bbR^d$
\begin{align}\label{eq:bounded-variance}
        \bbE_{\xi_i\sim p_i}\left\lVert\nabla \ell_i(w,\xi_i) - \nabla f_i(w)\right\rVert^2\leq \sigma_g^2.
\end{align}
\end{assumption}
Assumption \ref{assump:bounded-variance} is standard in the analysis of \mtext{SGD}-based methods and has been used in many relevant works \cite{stich2019local,nguyen2022federated,khaled2020tighter,wang2020tackling,koloskova2019decentralized2,koloskova2022sharper,toghani2022parspush}.
Since we perform updates using data batches, we also need to show the stochastic variance for the sampled batches. Recall that for simplicity; we assumed that all batch sizes are larger than $b\geq 1$, thus, we have:
\begin{align}\label{eq:bounded-variance-batch}
\bbE_{p_i}\left\lVert\nabla\tilde{f}_i(w,\mcD_i)-\nabla f_i(w)\right\rVert^2 \leq \frac{\sigma_g^2}{|\mcD_i|}\leq \sigma_a^2 \coloneqq \frac{\sigma_g^2}{b}
\end{align}

Next, we present the bounded heterogeneity assumption.

\begin{assumption}[Bounded Population Diversity]\label{assump:bounded-heterogeneity} For all $w\in\bbR^d$, the gradients of local functions $f_i(w)$ and the global function $f(w)$ satisfy the following property:
\begin{align}\label{eq:bounded-heterogeneity}
        \frac{1}{n}\sum\limits_{i=1}^{n}\lVert\nabla f_i(w) - \nabla f(w)\rVert^2\leq \gamma_g^2.
\end{align}
\end{assumption}
The above assumption measures the population diversity (heterogeneity) between the gradients. In heterogeneous settings, this bound indicates the similarity between different distributions. \citet{fallah2020personalized} show connections between heterogeneity and the Wasserstein distance between the distributions under certain assumptions.

The above assumptions are sufficient to prove the convergence of our method (Algorithms \ref{alg:server} \& \ref{alg:client}) under \textcolor{royalblue}{Option A} and \textcolor{seagreen}{Option C}. Therefore, we present the convergence analyses starting from our results on \mtext{FedAsync}.

\subsection{Asynchronous Federated Learning (\textcolor{royalblue}{Option A})}\label{subsec:afl}
We now demonstrate the convergence rate of our method for the cost function in \eqref{eq:fl}.

\begin{theorem}[\mtext{FedAsync}]\label{thm:afl}
Let Assumptions \ref{assump:staleness}-\ref{assump:bounded-heterogeneity} hold, $\beta=1$, and $\eta=\frac{1}{Q\sqrt{LT}}$. Then, the following property holds for the joint iterates of Algorithms \ref{alg:server} {\normalfont\&} \ref{alg:client} under \textcolor{royalblue}{Option A} on Problem \eqref{eq:fl}: for any timestep \mbox{$T\geq 160L(Q{+}7)(\tau{+}1)^3$} at the server
\begin{align*}
\frac{1}{T}\sum_{t=0}^{T-1} \,\bbE\left\lVert\nabla f\left(w^t\right)\right\rVert^2 &\leq \frac{4\sqrt{L}\left(f(w^0)-f^\star\right)}{\sqrt{T}} + \frac{8\sqrt{L}\left(\frac{\sigma_g^2}{b} +\gamma_g^2\right)}{\sqrt{T}}\\
&+\frac{80 L(1{+}Q)(\tau^2{+}1)\left(\frac{\sigma_g^2}{b} +\gamma_g^2\right)}{T}.
\end{align*}
\end{theorem}
The proof of Theorem \ref{thm:afl} is provided in Appendix \ref{app:afl-proof}. This theorem suggests a convergence rate of \mbox{$\mcO\left(\frac{1}{\sqrt{T}}\right) + \mcO\left(\frac{Q\tau^2}{T}\right)$} for asynchronous federated learning \mtext{FedAsync}. Our analysis removes the unnecessary boundedness assumption on the gradient norm.

\begin{remark}
Selecting $\beta=1$ in Theorem \ref{thm:afl}, results in a sub-optimal first-order stationary rate for smooth non-convex cost functions. However, this is an arbitrary choice for the value of $\beta$ and can be relaxed to any $\beta = \mcO(1)$ similar to \cite{nguyen2022federated}.
\end{remark}

Next, we present the convergence of \mtext{PersA-FL-MAML} along with some technical lemmas borrowed from \cite{fallah2020personalized}.

\subsection{Personalized Asynchronous Federated Learning: Model-Agnostic Meta-Learning Setup (\textcolor{brickred}{Option B})}\label{subsec:persafl-maml}

As we discussed in Section \ref{sec:algorithm}, we require the second-order derivatives of the local functions to compute the gradients of the personalized costs in \eqref{eq:persafl-maml}. Accordingly, we consider similar assumptions for the second-order derivatives as Assumptions \ref{assump:smoothness}-\ref{assump:bounded-heterogeneity}.

\begin{assumption}[Second-Order Properties]\label{assump:second-order}
For all clients $i\in[n]$, the following properties hold for the Hessian of each $f_i:\bbR^d\to\bbR$, the variance of a stochastic Hessian $\nabla^2 \ell_i(w,\xi_i)$ on a single data point $\xi_i\in\mcS_i$, and the global Hessian $\nabla^2 f(w)$: for all $w,u\in\bbR^d$,
\begin{align}
    &\left\lVert\nabla^2 f_i(w)-\nabla^2 f_i(u)\right\rVert\leq \rho\lVert w-u\rVert,\label{eq:hessian-lipschitz}\\
    \bbE_{\xi_i\sim p_i}&\left\lVert\nabla^2 \ell_i(w,\xi_i) - \nabla^2 f_i(w)\right\rVert^2\leq \sigma_h^2,\label{eq:hessian-bounded-variance}\\
    \frac{1}{n}\sum\limits_{i=1}^{n}&\left\lVert\nabla^2 f_i(w) - \nabla^2 f(w)\right\rVert^2\leq \gamma_h^2.\label{eq:hessian-bounded-heterogeneity}
\end{align}
\end{assumption}

Assumption \ref{assump:second-order} is conventional in the analysis of methods with access to second-order information \cite{fallah2020convergence,fallah2020personalized,safaryan2021fednl,toghani2022parspush}. Finally, we adopt another assumption from \cite{finn2019online,fallah2020personalized,fallah2021generalization} on the gradient norm to simplify the analysis for the \mtext{MAML} cost.

\begin{assumption}[Bounded-Gradient]\label{assump:bounded-gradient}
There exists a constant $G$ such that for all clients $i\in[n]$, and any parameter $w\in\bbR^d$, 
\begin{align}
   \left\lVert\nabla f_i(w)\right\rVert\leq G.
\end{align}
\end{assumption}
To the best of our knowledge, seminal works on \mtext{MAML} loss mainly consider this assumption to simplify the properties of the personalized function. Note that we consider Assumptions \ref{assump:second-order}-\ref{assump:bounded-gradient} \underline{only} in the analysis of \mtext{PersA-FL} (Algorithms \ref{alg:server} \& \ref{alg:client}) under \underline{\textcolor{brickred}{Option B}}. Under Assumptions \ref{assump:smoothness} and \ref{assump:bounded-gradient}, the properties in \eqref{eq:hessian-bounded-heterogeneity} and \eqref{eq:bounded-heterogeneity} can be simply derived with $\gamma_h = 2L $ and $\gamma_g = 2G$ \cite{fallah2020personalized}.

Before stating the convergence of \mtext{PersA-FL-MAML}, let us state some technical lemmas on the personalized \mtext{MAML} cost function.
\begin{lemma}[\cite{fallah2020personalized}, Lemma 4.2 - Smoothness: \mtext{MAML}]\label{lem:smoothness-maml}
Let Assumptions \ref{assump:smoothness} and \ref{assump:bounded-gradient} hold. Then, $F^{(b)}_i$ in \eqref{eq:persafl-maml} is $L_b$-smooth, i.e., for all clients $i\in[n]$, and any parameters $w,u\in\bbR^d$,
\begin{align}\label{eq:smoothness-maml}
        \left\lVert\nabla F^{(b)}_i(w)-\nabla F^{(b)}_i(u)\right\rVert\leq L_b\lVert w-u\rVert,
\end{align}
where $L_b\coloneqq L(1{+}\alpha L)^2 + \alpha \rho G$.
\end{lemma}
Lemma \ref{lem:smoothness-maml} indicates that the personalized cost in \eqref{eq:persafl-maml} is also smooth. The smoothness parameter $L_b$ depends on the personalization hyperparameter $\alpha$. Increasing the value of $\alpha$ results in higher smoothness constant $L_b$. The smoothness property of \mtext{MAML} cost under multi-step personalization (instead of one) is shown in \cite{toghani2022parspush}[Lemma 3].

\begin{lemma}[\cite{fallah2020personalized}, Lemma 4.3 - Bounded Variance: \mtext{MAML}]\label{lem:bounded-variance-maml}
Let Assumptions \ref{assump:smoothness}, \ref{assump:bounded-variance}, \ref{assump:second-order}, and \ref{assump:bounded-gradient} hold, and data batches $\mcD, \mcD', \mcD''$ be randomly sampled according to data distribution $p_i$. Then, the following properties hold for the stochastic personalized gradient $\nabla\tilde{F}^{(b)}_i(w,\mcD'',\mcD',\mcD)$:
\begin{align}
\left\lVert\bbE_{p_i}\left[\nabla\tilde{F}^{(b)}_i(w,\mcD'',\mcD',\mcD) - \nabla F^{(b)}_i(w)\right]\right\rVert&\leq\mu_b\coloneqq\frac{\alpha L(1{+}\alpha L)\sigma_g}{\sqrt{b}},\label{eq:unbiased-mean-maml}\\
\bbE_{p_i}\left\lVert\nabla\tilde{F}^{(b)}_i(w,\mcD'',\mcD',\mcD) - \nabla F^{(b)}_i(w)\right\rVert^2&\leq \sigma_b^2,\label{eq:bounded-variance-maml}
\end{align}
for all $w\in\bbR^d$, where \mbox{$\sigma_b^2\coloneqq 3(1{+}\alpha L)^2 \sigma_g^2 \left[\frac{1}{b} {+} \frac{\alpha^2 L^2}{b}\right] + 3\alpha^2 G^2 \frac{\sigma_h^2}{b} + \frac{3\alpha^2\sigma_g^2\sigma_h^2}{b}\left[\frac{1}{b} {+} \frac{\alpha^2 L^2}{b}\right]$}.
\end{lemma}
Lemma \ref{lem:bounded-variance-maml} highlights two important results. First, the stochastic gradient in \eqref{eq:persafl-maml-stoch-grad} is a biased estimation of the full gradient \ref{eq:persafl-maml-full-grad}. The biasness is controlled by two factors, personalization stepsize $\alpha$, and batch size $b$.\footnote{It should be noted that the batch size in the upper bound of \eqref{eq:unbiased-mean-maml} refers to the size of $|\mcD'|$. Recall that we use this batch to approximate the inner gradient in \ref{eq:persafl-me-stoch-grad}.} Therefore, we obtain an unbiased estimation under no personalization, i.e., $\alpha=0$. However, as we select a larger $\alpha$, we require more samples to reduce the error imposed by biased gradient estimations. Second, similar to Assumption \ref{assump:bounded-variance} on the cost; we have a tight variance based on $\alpha$ and $b$.


\begin{lemma}[\cite{fallah2020personalized}, Lemma 4.4 - Bounded Population Diversity: \mtext{MAML}]\label{lem:bounded-heterogeneity-maml}
For all $w\in\bbR^d$, the gradients of local personalized functions $F^{(b)}_i(w)$ and the global function $F^{(b)}(w)$ satisfy the following property:
\begin{align}\label{eq:bounded-heterogeneity-maml}
        \frac{1}{n}\sum\limits_{i=1}^{n}\left\lVert\nabla F^{(b)}_i(w) - \nabla F^{(b)}(w)\right\rVert^2\leq \gamma_b^2\coloneqq 12(1+\alpha L)^2 \left[1+\alpha^2L^2\right]\gamma_g^2 + 12 \alpha^2G^2\gamma_h^2.
\end{align}
\end{lemma}
The above lemma determines the heterogeneity of the personalized gradients $\nabla F_i^{(b)}(w)$ based on the heterogeneity of gradient and Hessian. One can see the connection of this bound with $\mcO(\gamma_g^2) +\alpha^2\mcO(\gamma_h^2)$, whereby setting $\alpha=0$, we recover the same heterogeneity in terms of $\mcO(\cdot)$ notion.


\begin{lemma}[Bounded-Gradient: \mtext{MAML}]\label{lem:bounded-gradient-maml}
For all clients $i\in[n]$, and any parameter $w\in\bbR^d$,
\begin{align}\label{eq:bounded-gradient-maml}
   \left\lVert\nabla F^{(b)}_i(w)\right\rVert\leq G_b \coloneqq (1{+}\alpha L)G.
\end{align}
\end{lemma}
This lemma indicates that the bound on the norm of personalized gradients potentially increases under a larger personalization budget $\alpha$.

Building upon the results in Lemmas \ref{lem:bounded-variance-maml}-\eqref{eq:bounded-gradient-maml}, we are now ready to present the convergence result for \mtext{PersA-FL-MAML}.

\begin{theorem}[\mtext{PersA-FL-MAML}]\label{thm:persafl-maml}
Let Assumptions \ref{assump:staleness}-\ref{assump:bounded-gradient} hold, \mbox{$\alpha\geq 0$}, \mbox{$\beta=1$}, and \mbox{$\eta=\frac{1}{Q\sqrt{L_bT}}$}. Then, the following property holds for the joint iterates of Algorithms \ref{alg:server} {\normalfont\&} \ref{alg:client} under \textcolor{brickred}{Option B} on Problem \eqref{eq:persafl-maml}: for any timestep \mbox{$T\geq 64L_b$} at the server
\begin{align*}
\frac{1}{T}\sum\limits_{t=0}^{T{-}1}\bbE\left\lVert\nabla F^{(b)}(w^{t})\right\rVert^2
&\leq \frac{4\sqrt{L_b}\left(F^{(b)}(w^{0})- f^\star\right)}{\sqrt{T}} + \frac{8\sqrt{L_b}\left(\sigma_b^2 + \gamma_b^2\right)}{\sqrt{T}}\\
&+ \frac{20\,Q L_b \left(G_b^2 {+} \sigma_b^2\right)\left(\tau^2{+}1\right)}{T} + \frac{4\,Q\alpha^2 L^2(1{+}\alpha L)^2 \sigma_g^2}{b}.
\end{align*}
\end{theorem}

The proof of this theorem can be found in Appendix \ref{app:persafl-maml}.
Theorem \ref{thm:persafl-maml} shows a convergence rate of \mbox{$\mcO\left(\frac{1}{\sqrt{T}}\right) + \mcO\left(\frac{\tau^2}{T}\right) + \mcO\left(\frac{\alpha^2 \sigma_g^2}{b}\right)$} for \mtext{PersA-FL} algorithm under \mtext{MAML} setup. Now, let us compare this rate with the convergence rate of \mtext{FedAsync} and \mtext{Per-FedAvg}, as in Table \ref{tab:comparison}. The last term in the above rate, i.e., $\mcO\left(\frac{\alpha^2 \sigma_g^2}{b}\right)$ accounts for personalization with biased gradient estimation. Moreover, compared to \mtext{Per-FedAvg}, the second term of this rate is different, which accounts for the maximum delay in asynchronous updates. 

To achieve the optimal complexity bound for the result in Theorem \ref{thm:persafl-maml}, we show how to choose the parameters $T,b$ based on the desired accuracy $\varepsilon$ in the following corollary.

\cu{
\begin{corollary}[\mtext{PersA-FL-MAML} $\varepsilon$-convergence]\label{cor:persafl-maml} Suppose the conditions in Theorem \ref{thm:persafl-maml} are satisfied. Algorithms \ref{alg:server} {\normalfont\&} \ref{alg:client} under \textcolor{brickred}{Option B} finds an $\varepsilon$ first-order stationary solution for $F^{(b)}$ in \eqref{eq:persafl-maml} by setting $T=\mcO(\varepsilon^{-2})$ and $b=\mcO(\varepsilon^{-1})$ given a fixed personalization budget $\alpha\geq 0$.
\end{corollary}

The result in Corollary \ref{cor:persafl-maml} highlights the required communication and sample complexity for $\varepsilon$ first-order stationary convergence. Moreover, note that the last expression in Theorem \ref{thm:persafl-maml} can also be controlled through a combined stepsize $\alpha$ and batch size $b$. This result is consistent with intuition, i.e., more samples are required to obtain a higher degree of personalization.
}

Next, we will present the analysis of \mtext{PersA-FL-ME}.

\subsection{Personalized Asynchronous Federated Learning: Moreau Envelope Setup (\textcolor{seagreen}{Option C})}\label{subsec:persafl-me}

In this subsection, we show three technical lemmas on the bounded variance of stochasticity and heterogeneity  as well as smoothness for \mtext{ME} formulation \eqref{eq:persafl-me} and then present the convergence rate of \mtext{PersA-FL} for this personalization framework. The proof of all results in this subsection is provided in Appendix \ref{app:persafl-me}.

First, we present the smoothness property of \mtext{ME} loss.

\begin{lemma}[Smoothness: ME]\label{lem:smoothness-me}
Let Assumption \ref{assump:smoothness} holds and \mbox{$\lambda\geq \kappa L$} for some \mbox{$\kappa > 1$}. Then, $F^{(c)}_i$ in \eqref{eq:persafl-me} is $L_c$-smooth, where \mbox{$L_c=\frac{\lambda}{\kappa{-}1}$}.
\end{lemma}

According to Lemma \ref{lem:smoothness-me}, we limit our exploration to $\lambda> L$ which satisfies the smoothness constraint for the \mtext{ME} formulation. In fact, according to Appendix \ref{app:persafl-me}, one can also see that originally, each $F^{(c)}_i(\cdot)$ is $\frac{\lambda L}{\lambda{-}L}$-smooth which is also smaller than $L_c=\frac{\lambda}{\kappa{-}1}$. As we mentioned in Section \ref{sec:setup}, when $\lambda\to\infty$, \mtext{ME} framework converts to \mtext{FL}. The smoothness property in Lemma \ref{lem:smoothness-me} is tight because, $L_c\to L$ if $\lambda\to\infty$.

\begin{corollary}[\cite{dinh2020personalized}, Proposition 1]
If $\lambda\geq 2L$, then Lemma \ref{lem:smoothness-me} implies that $F^{(c)}_i$ in \eqref{eq:persafl-me} is $\lambda$-smooth.
\end{corollary}

\begin{lemma}[Bounded Variance: ME]\label{lem:bounded-variance-me}
Let Assumptions \ref{assump:smoothness} and \ref{assump:bounded-variance} hold, $\lambda \geq \kappa L$ (for some $\kappa> 1$), and the data batch $\mcD$ be randomly sampled according to data distribution $p_i$. Then, the following properties hold for the stochastic personalized gradient $\nabla\tilde{F}^{(c)}_i(w,\mcD)$: for all $w\in\bbR^d$,
\begin{align}
\left\lVert\bbE_{p_i}\left[\nabla\tilde{F}^{(c)}_i(w,\mcD) - \nabla F^{(c)}_i(w)\right]\right\rVert&\leq\mu_c\coloneqq\frac{\lambda}{\lambda{-}L}\nu,\label{eq:unbiased-mean-me}\\
\bbE_{p_i}\left\lVert\nabla\tilde{F}^{(c)}_i(w,\mcD) - \nabla F^{(c)}_i(w)\right\rVert^2&\leq \sigma_c^2\coloneqq \frac{2\lambda^2}{\left(\lambda{-}L\right)^2}\left[\frac{\sigma_g^2}{b}+\nu^2\right].\label{eq:bounded-variance-me}
\end{align}
\end{lemma}
This lemma is analogous to Lemma \ref{lem:bounded-variance-maml} in Subsection \ref{subsec:persafl-maml}. In Lemma \ref{lem:bounded-variance-me}, we show an upper bound on the variance and bias of the stochastic gradient compared to the full gradient. Note that when $\lambda\to \infty$, we know that $\hat{\theta}_i(w) \to w$. Therefore, by fixing $\tilde{\theta}_i(w)=w$, it is guaranteed that $\nu=0$, thus our gradient estimation becomes unbiased and the variance similar to \eqref{eq:bounded-variance-batch}.

\begin{lemma}[Bounded Population Diversity: ME]\label{lem:bounded-heterogeneity-me} Let personalization hyperparameter $\lambda \geq 7L$. Then, for all $w\in\bbR^d$, the gradients of local personalized functions $F^{(c)}_i(w)$ and the global \mtext{ME} function $F^{(c)}(w)$ satisfy the following property:
\begin{align}\label{eq:bounded-heterogeneity-me}
        \frac{1}{n}\sum\limits_{i=1}^{n}\left\lVert\nabla F^{(c)}_i(w) - \nabla F^{(c)}(w)\right\rVert^2\leq \gamma_c^2\coloneqq \frac{16\lambda^2}{\lambda^2{-}48L^2}\gamma_g^2.
\end{align}
\end{lemma}
Lemma \ref{lem:bounded-heterogeneity-me} provides a bound on population diversity of \mtext{ME} as a factor of $\gamma_g^2$. Similar to what we explained so far, for $\lambda\to\infty$, the heterogeneity bound turns into $\gamma_g^2$.

\begin{remark}
In the analysis for Theorem \ref{thm:persafl-me}, we consider bounded population diversity as in Assumption \ref{assump:bounded-heterogeneity}, average bounded diversity. \cite{dinh2020personalized}[Assumption 3] and \cite{wang2021field}[6.1.1  Assumptions and Preliminaries, (vii)] consider a slightly stronger version of this assumption, namely uniformly ``bounded heterogeneity'' which is defined as follows:
\begin{align}\label{eq:uniformly-bounded-heterogeneity}
        \max_{i\in[n]}\sup_{w\in\bbR^d}\lVert\nabla f_i(w) - \nabla f(w)\rVert^2\leq \gamma_g^2.
\end{align}
Under the modified assumption in \eqref{eq:uniformly-bounded-heterogeneity}, we can improve  $\gamma_c^2\coloneqq \frac{16\lambda^2}{\lambda^2{-}8L^2}\gamma_g^2$.
\end{remark}

Now, we present our convergence result of \mtext{PersA-FL-ME} under Assumption \ref{assump:staleness}-\ref{assump:bounded-heterogeneity}
\begin{theorem}[\mtext{PersA-FL-ME}]\label{thm:persafl-me}
Let Assumptions \ref{assump:staleness}-\ref{assump:bounded-heterogeneity} hold, $\lambda\geq 7L$, $\beta=1$, and \mbox{$\eta=\frac{1}{Q\sqrt{L_c T}}$}. Then, the following property holds for the joint iterates of Algorithms \ref{alg:server} {\normalfont\&} \ref{alg:client} under \textcolor{seagreen}{Option C} on Problem \eqref{eq:persafl-me}: for any timestep \mbox{$T\geq 288 L_c(Q{+}7)(\tau{+}1)^2$} at the server
\begin{align*}
\frac{1}{T}\sum\limits_{t=0}^{T{-}1}\bbE\left\lVert\nabla F^{(c)}(w^{t})\right\rVert^2
&\leq \frac{4\sqrt{L_c}\left(F^{(c)}(w^{t})- f^\star\right)}{\sqrt{T}} + \frac{8\sqrt{L_c}\left(\sigma_c^2 + \gamma_c^2\right)}{\sqrt{T}}\\
&+\frac{144 L_c(1{+}Q)(\tau^2{+}1)\left(\sigma_c^2 +\gamma_g^2\right)}{T} + \frac{4Q\lambda^2\nu^2}{(\lambda{-}L)^2}.
\end{align*}
\end{theorem}
Theorem \ref{thm:persafl-me} proposes a convergence rate of \mbox{$\mcO\left(\frac{1}{\sqrt{T}}\right) + \mcO\left(\frac{\tau^2}{T}\right) + \mcO\left(\frac{\lambda^2\nu^2}{(\lambda{-}L)^2}\right)$} for \mtext{PersA-FL} under \mtext{ME} formulation. Again, under the exact same reasoning as Lemma \ref{lem:bounded-variance-me}, we know that $\nu=0$ when $\lambda \to \infty$, thus the convergence rate simply reduces to \mtext{FedAsync} with no personalization. \cu{Moreover, let us compare the convergence result in Theorem \ref{thm:persafl-me} with the rate of \mtext{pFedMe} \cite{dinh2020personalized} in Table \ref{tab:comparison}. By comparing the last terms in both rates, $\mcO\Big(\frac{\lambda^2\nu^2}{(\lambda{-}L)^2}\Big)$ and $\mcO\Big(\frac{\lambda^2\left(\frac{1}{b}+\nu^2\right)}{(\lambda{-}L)^2}\Big)$, one can see that the additional term $\frac{1}{b}$ in the convergence rate of \mtext{pFedMe}, implies that even under $\lambda \to \infty$ (i.e., no personalization), the last term does not vanish unless we select large data batches, i.e., $b=\mcO(\varepsilon^{-1})$. Therefore, from the personalization perspective, our analysis provides a tighter bound compared to \mtext{pFedMe}.}

\cu{
In the next corollary, we characterize a choice of $T,\nu$ in Theorem \ref{thm:persafl-me}, given a desired accuracy level $\varepsilon$ for our proposed algorithm in \textcolor{seagreen}{Option C}.

\begin{corollary}[\mtext{PersA-FL-ME} $\varepsilon$-convergence]\label{cor:persafl-me} Suppose the conditions in Theorem \ref{thm:persafl-me} are satisfied. Algorithms \ref{alg:server} {\normalfont\&} \ref{alg:client} under \textcolor{seagreen}{Option C} finds an $\varepsilon$ first-order stationary solution for $F^{(c)}$ in \eqref{eq:persafl-me} by setting $T=\mcO(\varepsilon^{-2})$ and $\nu=\mcO(\varepsilon^{{1}/{2}})$ given a fixed personalization budget $\lambda$.
\end{corollary}

Corollary \ref{cor:persafl-maml} determines the communication complexity and precision of the approximate gradient estimator to achieve an $\varepsilon$ first-order stationary convergence. This means that if we choose $\nu=\mcO(\varepsilon^{{1}/{2}})$, then the inexact optimization solver should compute the solution up to accuracy $O(\nu)$ of the surrogate optimization problem in order to achieve an $\varepsilon$-first order stationary solution. Also, we would like to highlight that this result implies no direct dependence on the batch size (b) for the convergence result of our algorithm with \textcolor{seagreen}{Option C}(cf. \cite{dinh2020personalized}).
}

\section{Numerical Experiments}\label{sec:experiments}

\cu{
In this section, we evaluate the performance of our method in settings with delayed communications. We focus on the aspects of concurrency, speed-up, and accuracy.}

Let us first start by explaining our simulation setup for communications with delays. We consider a set of $n{=}30$ different clients. \cu{Each of the clients has a set of random delays at the upload and download stage. The random delays are generated such that the average upload delay is $4$ to $6$ times higher than the average download delay. Moreover, we assume that the time for communication and aggregation is much larger than the time for local updates, thus we focus on the communication time. First, we show the number of active (not idle) users during the training process under asynchronous communications. The orange curve in Figure \ref{fig:results}\textbf{(a)} shows the proportion of active users, which is up to $80\%$ on average over time. We also plot the average proportion of users sampled in the synchronous updates in the same figure with green color. As Figure \ref{fig:results}\textbf{(a)} demonstrates, the concurrency level for asynchronous methods is considerably higher than that of their synchronous counterparts.}

\cu{We create extremely heterogeneous distributed data from MNIST \cite{lecun1998mnist} and CIFAR-10 \cite{krizhevsky2010convolutional} datasets for the clients, meaning that each client holds a different and skewed distribution of images from various classes. To build the heterogeneous data, we assign each client $i\in[n]$ samples from only $c$ out of $10$ classes of the data.} Over the underlying communication setup and heterogeneous data setting, we compare the speed and accuracy of \mtext{FedAvg}, \mtext{Per-FedAvg}, \mtext{pFedMe}, \cu{\mtext{SCAFFOLD}}\footnote{\cu{Scaffold algorithm has two options. Option I makes another pass over the local data to compute the gradient at the server model. Therefore, we consider SCAFFOLD (Option I), which is more stable in practice \cite{karimireddy2020scaffold}.}}, \mtext{FedAsync}, \mtext{\textcolor{brickred}{PersA-FL-MAML}}, \mtext{\textcolor{seagreen}{PersA-FL-ME}}, where the first \cu{four} methods are synchronous and the rest are asynchronous. For MNIST and CIFAR-10, we consider convolutional networks \cite{krizhevsky2010convolutional} followed by fully connected layers with pooling and dropout as well as cross-entropy loss. \cu{Details on the experimental setups can be found in Appendix\ref{app:exp-setup}.}

Figure \ref{fig:results} \textbf{(b)} and Figure \ref{fig:results} \textbf{(c)} compares the performance and convergence speed of our methods (\mtext{\textcolor{brickred}{PersA-FL-MAML}} \& \mtext{\textcolor{seagreen}{PersA-FL-ME}}) \cu{with the other five algorithms respectively on heterogeneous MNIST and CIFAR-10 datasets}. \cu{\emph{We would like to emphasize that in these two figures, each point on each curve represents the accuracy of the corresponding method after local fine-tuning with the same personalization budget as personalized personalized algorithm.} In other words, similar to the four personalized methods, \mtext{Per-FedAvg}, \mtext{pFedMe}, \mtext{\textcolor{brickred}{PersA-FL-MAML}}, and \mtext{\textcolor{seagreen}{PersA-FL-ME}}, we consider same amount of fine-tuning budget for the three non-personalized methods. \mtext{FedAvg}, \mtext{SCAFFOLD}, and \mtext{FedAsync}. As shown in Figure \ref{fig:results}, our methods outperform the other methods within a fixed communication time. Moreover, the \mtext{ME} loss function results in a more stable and efficient performance compared to \mtext{MAML}.}

\begin{figure}
    \centering
    \includegraphics[width=0.304\linewidth]{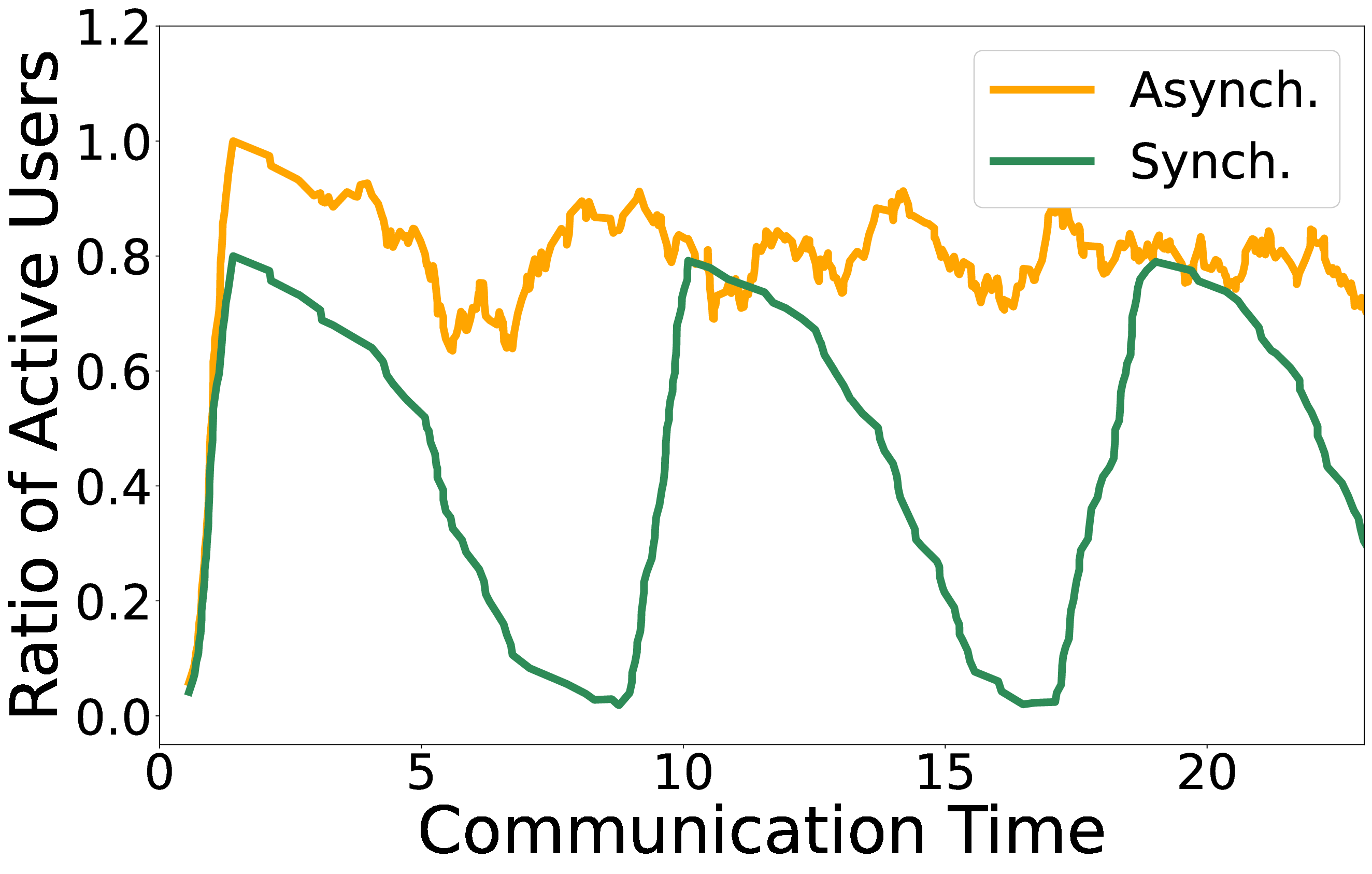}
    \includegraphics[width=0.32\linewidth]{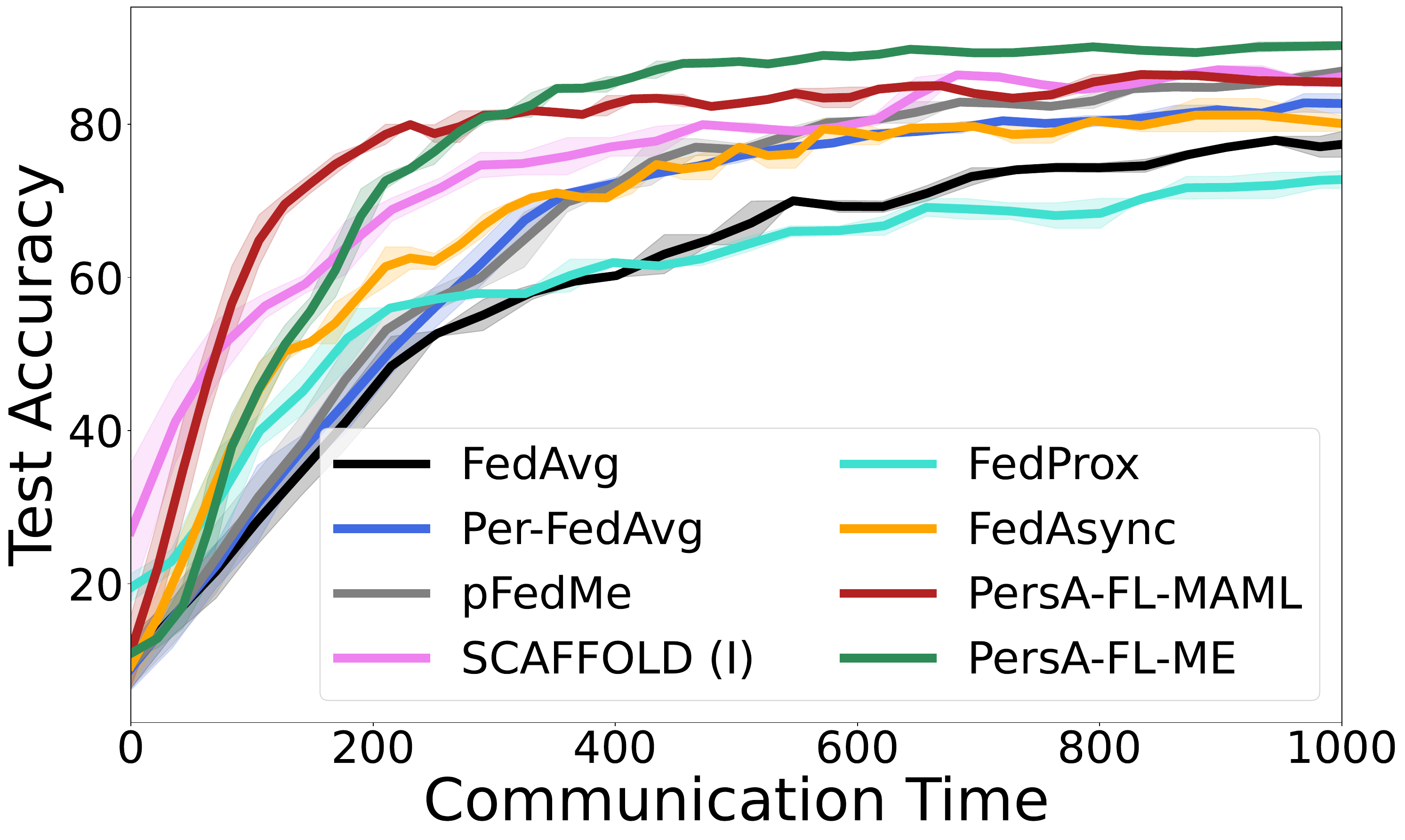}
    \includegraphics[width=0.324\linewidth]{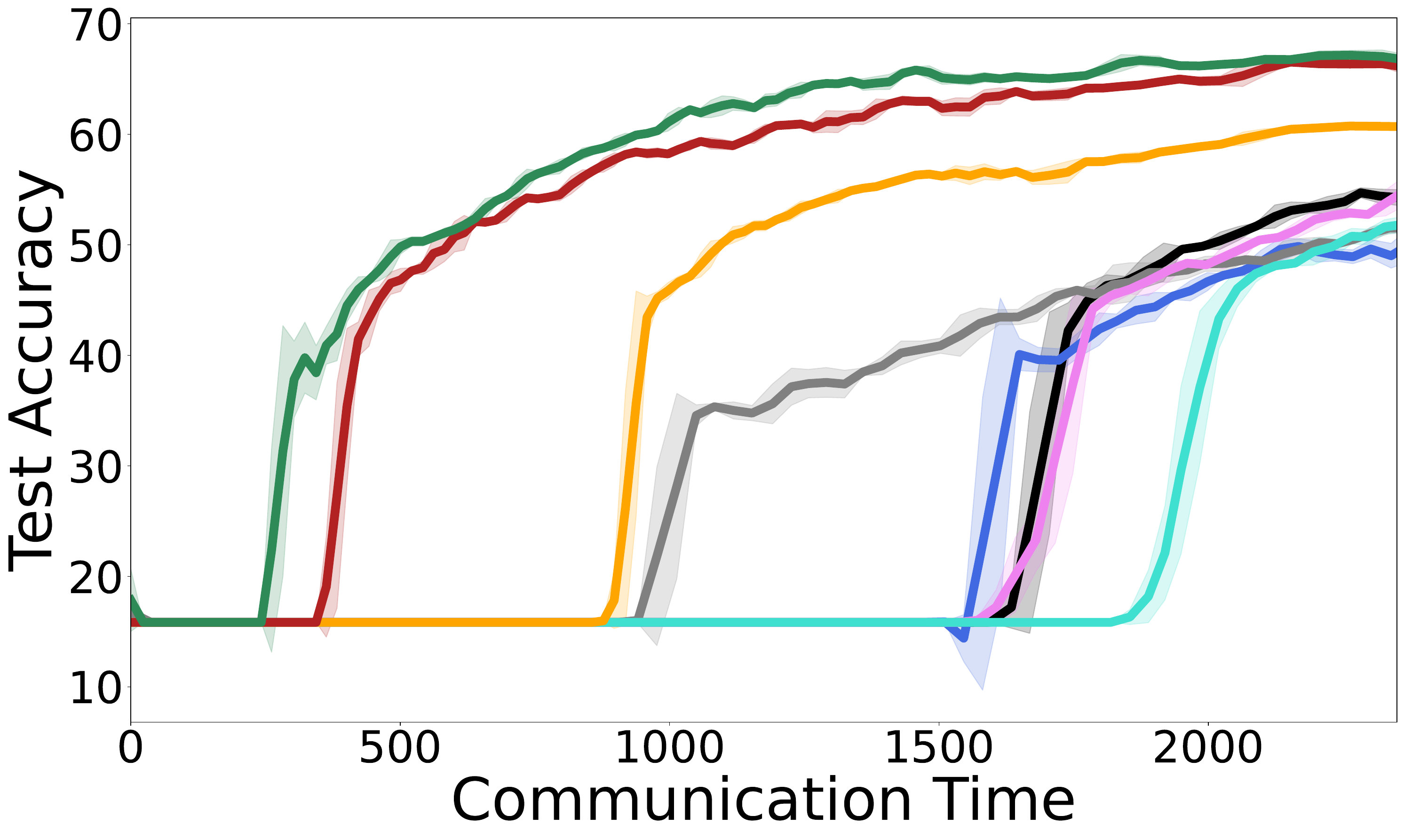}
    
    \textbf{(a) Concurrency}\hspace{4.5em} \textbf{(b) MNIST} \hspace{5.5em} \textbf{(c) CIFAR-10}
    \caption{The impact of heterogeneity and communication delays on concurrency, convergence speed, and performance of multiple \mtext{FL}-based algorithms. The underlying setup of this experiment consists of $n{=}30$ clients, $Q{=}10$ local updates, and each client has a random upload and download delay at each round. \textbf{(a)} A comparison between the ratio of active users for synchronous and asynchronous updates over the course of training. \textbf{(b)} Comparison between the test accuracy of \mtext{FedAvg}, \mtext{Per-FedAvg}, \mtext{pFedMe}, \cu{\mtext{SCAFFOLD}, \mtext{FedProx}}, \mtext{FedAsync}, \mtext{PersA-FL-MAML}, and \mtext{PersA-FL-ME} on MNIST data with heterogeneous distribution. \textbf{(c)} Test accuracy of the mentioned methods on CIFAR-10 data with synthetic heterogeneity within a limited fixed time.}
    \label{fig:results}
\end{figure}

\section{Conclusion}\label{sec:conclusion}

This work studied the personalized federated learning problem for the heterogeneous data setting under asynchronous communications with the server. We considered the Model-Agnostic Meta-Learning (\mtext{MAML}) and Moreau Envelope (\mtext{ME}) formulations to account for personalization. We proposed the \mtext{PersA-FL} algorithm to solve this problem under stale updates. We showed the convergence rate of our method for smooth non-convex functions asynchronous federated learning, and personalized federated learning under the two personalization formulations, i.e., \mtext{MAML} and \mtext{ME}. Particularly, for asynchronous federated learning and personalized federated learning under Moreau Envelope costs, we presented a proof technique that does not require a boundedness assumption on the gradient norm. We finally show numerical results that illustrate the benefits of our proposed method in terms of accuracy and scalability. The studies of generalization and communication efficiency will be left to future research. Moreover, the extensions of our method to the buffered aggregation and decentralized setups remain for future studies.

\clearpage
\bibliographystyle{plainnat}
\bibliography{ref}

\clearpage
\appendix
\section{Asynchronous Federated Learning}\label{app:afl-proof}

\begin{proof}[Proof of Theorem \ref{thm:afl}]
First, we present a set of useful inequalities we will use in the proof. For any set of $m$ vectors \mbox{$\{w_i\}_{i{=}1}^m$} such that \mbox{$w_i\in\bbR^{d}$}, and a constant \mbox{$\alpha>0$}, the following properties hold: for all \mbox{$i,j\in[m]$:}
\begingroup
\allowdisplaybreaks
\begin{subequations}\label{eq:gen-ineq}
\begin{align}
    \lVert w_i + w_j \rVert^2 &\leq (1{+}\alpha)\lVert w_i\rVert^2 + (1{+}\alpha^{-1})\lVert w_j\rVert^2,\label{eq:gen-ineq-1}
    \\\lVert w_i + w_j \rVert &\leq \lVert w_i\rVert + \lVert w_j\rVert,\label{eq:gen-ineq-2}
    \\2\langle w_i, w_j\rangle &\leq \alpha\lVert w_i\rVert^2 + \alpha^{-1}\lVert w_j\rVert^2,\label{eq:gen-ineq-3}
    \\\left\lVert\sum\limits_{i=1}^m  w_i\right\rVert^2 &\leq m \left(\sum\limits_{i=1}^m \lVert w_i\rVert^2\right)\label{eq:gen-ineq-4}
    .
\end{align}
\end{subequations}
\endgroup

Now, let us rewrite the update rule of the joint iterates in Algorithms \ref{alg:server} \& \ref{alg:client} \mbox{\textcolor{royalblue}{Option A}} at time $t$ as follows:
\begin{itemize}
    \item Client update:
    \begin{align}
        w_{i,0}^t &= w^t,\\
        w_{i,q{+}1}^t &= w_{i,q}^t - \eta \nabla \tilde{f}_i(w_{i,q}^t, \mcD_{i,q}^t),
    \end{align}
    \item Server update:
    \begin{align}
         w^{t{+}1} = w^{t} - \beta \Delta_{i_t} = w^{t} - \eta\beta \sum\limits_{q=0}^{Q{-}1} \nabla \tilde{f}_{i_t}\left(w_{i_t,q}^{\Omega(t)},\mcD_{i,q}^{\Omega(t)}\right).
    \end{align}
\end{itemize}
For simplicity, we denote $\tilde{\nabla}f_{i}\left(w\right) = \nabla \tilde{f}_{i}\left(w, \mcD_i\right)$. Therefore, at round $t$, the server updates its parameter by receiving $\Delta_{i_t}$ from some client $i_t\in[n]$, as follows:
\begin{align}\label{eq:afl-update-rule-delay}
    w^{t{+}1} = w^{t} - \eta\beta \sum\limits_{q=0}^{Q{-}1} \tilde{\nabla} f_{i_t}\left(w_{i_t,q}^{\Omega(t)}\right).
\end{align}
Moreover, 
Due to Assumption \ref{assump:smoothness}, we can infer that $f$ is $L$-smooth, thus
\begin{align}\label{eq:afl-main-0}
    f\left(w^{t{+}1}\right) \overset{\eqref{eq:smoothness}}{\leq} f(w^{t}) -\eta\beta\underbrace{\left\langle \nabla f(w^{t}),  \sum\limits_{q=0}^{Q{-}1} \tilde{\nabla} f_{i_t}\left(w_{i_t,q}^{\Omega(t)}\right) \right\rangle}_{=: S_{a_1}} +\frac{L\eta^2\beta^2}{2}\underbrace{\left\lVert \sum\limits_{q=0}^{Q{-}1} \tilde{\nabla} f_{i_t}\left(w_{i_t,q}^{\Omega(t)}\right)\right\rVert^2}_{=: S_{a_2}}
\end{align}
First, we provide a lower bound on term $S_{a_1}$ in \eqref{eq:afl-main-0}. Prior to show the bound, let us denote \mbox{$\tilde{g}_i^t=\sum\limits_{q{=}0}^{Q{-}1} \tilde{\nabla}f_i\left(w_{i,q}^{\Omega(t)}\right)$}, \mbox{$\tilde{g}^t=\frac{1}{n}\sum\limits_{i{=}1}^n\tilde{g}_i^t$}, \mbox{$g_i^t=\sum\limits_{q{=}0}^{Q{-}1}\nabla f_i\left(w_{i,q}^{\Omega(t)}\right)$}, and \mbox{$g^t=\frac{1}{n}\sum\limits_{i{=}1}^n g_i^t$}. Therefore,
\begin{align}\label{eq:afl-s1}
\bbE\left[S_{a_1}\right] &= \bbE\left[\bbE_{i_t}\left\langle\nabla f(w^{t}), \tilde{g}_{i_t}^t\right\rangle\right]\\
&= \bbE\left[\left\langle \nabla f(w^{t}),  \frac{1}{n}\sum\limits_{i=1}^{n}\tilde{g}_{i}^t \right\rangle\right]\\ 
&= \bbE\left\langle \nabla f(w^{t}),  \frac{1}{n}\sum\limits_{i=1}^{n} \bbE_{p_i}\left[\tilde{g}_{i}^t\right]\right\rangle = \bbE\left\langle \nabla f(w^{t}),  \frac{1}{n}\sum\limits_{i=1}^{n} g_i^t\right\rangle\\
&= Q\,\bbE\left\lVert\nabla f(w^{t})\right\rVert^2 + \bbE\left\langle \nabla f(w^{t}),  g^t-Q\nabla f(w^{t})\right\rangle\\
\overset{\eqref{eq:gen-ineq-3}}&{\geq} Q\,\bbE\left\lVert\nabla f(w^{t})\right\rVert^2 - \frac{1}{2}\bbE\left\lVert \nabla f(w^{t}) \right\rVert^2 - \frac{1}{2}\bbE\left\lVert g^t-Q\nabla f(w^{t})\right\rVert^2\\
&= \frac{2Q{-}1}{2}\,\bbE\left\lVert\nabla f(w^{t})\right\rVert^2 - \frac{1}{2}\bbE\left\lVert g^t-Q\nabla f(w^{t})\right\rVert^2.\label{eq:afl-s1-end}
\end{align}
Moreover, the following holds for $S_{a_2}$ in \eqref{eq:afl-main-0}:
\begin{align}\label{eq:afl-s2}
\bbE_{i_t}\left[S_{a_2}\right] = \bbE_{i_t}\left\lVert\sum\limits_{q=0}^{Q{-}1} \tilde{\nabla} f_{i_t}\left(w_{i_t,q}^{\Omega(t)}\right)\right\rVert^2 = \frac{1}{n} \sum\limits_{i=1}^n \left\lVert \tilde{g}_i^t\right\rVert^2.
\end{align}
Now, according to \eqref{eq:afl-main-0}, \eqref{eq:afl-s1-end}, and \eqref{eq:afl-s2}, we have:
\begin{align}\label{eq:afl-main-1}
    \bbE f\left(w^{t{+}1}\right) &\leq \bbE f(w^{t}) -\frac{\eta\beta(2Q{-}1)}{2}\bbE\left\lVert\nabla f(w^{t})\right\rVert^2\\
    &+ \frac{\eta\beta}{2}\bbE\underbrace{\left\lVert g^t-Q\nabla f(w^{t})\right\rVert^2}_{=: S_{a_3}} +\frac{L\eta^2\beta^2}{2n}\bbE\underbrace{\left[\sum\limits_{i=1}^n \left\lVert \tilde{g}_i^t\right\rVert^2\right]}_{=: S_{a_4}},\label{eq:afl-main-1-end}
\end{align}
where we bound $S_{a_3}$ and $S_{a_4}$ as follows:
\begin{align}\label{eq:afl-s3}
S_{a_3} = \left\lVert \frac{1}{n}\sum\limits_{i=1}^n\left(g_i^t-Q\nabla f_i(w^{t})\right)\right\rVert^2
\overset{\eqref{eq:gen-ineq-4}}&{\leq} \frac{1}{n}\sum\limits_{i=1}^n \left\lVert g_i^t-Q\nabla f_i(w^{t})\right\rVert^2\\
&= \frac{1}{n}\sum\limits_{i=1}^n \left\lVert \sum\limits_{q{=}0}^{Q{-}1}\nabla f_i\left(w_{i,q}^{\Omega(t)}\right)-Q\nabla f_i(w^{t})\right\rVert^2\\
&= \frac{1}{n}\sum\limits_{i=1}^n \left\lVert \sum\limits_{q{=}0}^{Q{-}1}\left[\nabla f_i\left(w_{i,q}^{\Omega(t)}\right)-\nabla f_i(w^{t})\right]\right\rVert^2\\
\overset{\eqref{eq:gen-ineq-4}}&{\leq}\frac{Q}{n}\sum\limits_{i=1}^n\sum\limits_{q=0}^{Q{-}1} \left\lVert\nabla f_i\left(w_{i,q}^{\Omega(t)}\right)-\nabla f_i(w^{t})\right\rVert^2,\label{eq:afl-s3-end}
\end{align}

\begin{alignat}{2}\label{eq:afl-s4}
S_{a_4} &= \,\,\,\,\,\sum\limits_{i=1}^n \Big\lVert \sum\limits_{q{=}0}^{Q{-}1}\tilde{\nabla}&&f_i\left(w_{i,q}^{\Omega(t)}\right)\Big\rVert^2\\
\overset{\eqref{eq:gen-ineq-4}}&{\leq} \,\,\,\,Q\sum\limits_{i=1}^n\sum\limits_{q{=}0}^{Q{-}1} &&\Big\lVert\tilde{\nabla}f_i\left(w_{i,q}^{\Omega(t)}\right)\Big\rVert^2\\
&= \,\,\,\,\,\,\,\,Q\sum\limits_{i=1}^n\sum\limits_{q{=}0}^{Q{-}1} && \Big\lVert\tilde{\nabla}f_i\left(w_{i,q}^{\Omega(t)}\right) - \nabla f_i\left(w_{i,q}^{\Omega(t)}\right) + \nabla f_i\left(w_{i,q}^{\Omega(t)}\right) - \nabla f_i\left(w^t\right)\nonumber\\
&&& + \nabla f_i\left(w^t\right) - \nabla f\left(w^t\right) + \nabla f\left(w^t\right) \Big\rVert^2\\
\overset{\eqref{eq:gen-ineq-4}}&{\leq} 4Q\sum\limits_{i=1}^n\sum\limits_{q{=}0}^{Q{-}1}&&\Bigg[\Big\lVert\tilde{\nabla}f_i\left(w_{i,q}^{\Omega(t)}\right) - \nabla f_i\left(w_{i,q}^{\Omega(t)}\right)\Big\rVert^2 + \Big\lVert\nabla f_i\left(w_{i,q}^{\Omega(t)}\right) - \nabla f_i\left(w^t\right)\Big\rVert^2\nonumber\\
&&& + \Big\lVert\nabla f_i\left(w^t\right) - \nabla f\left(w^t\right)\Big\rVert^2 + \Big\lVert\nabla f\left(w^t\right) \Big\rVert^2 \Bigg]\label{eq:afl-s4-end} \Rightarrow
\end{alignat}

\begin{align}\label{eq:afl-s4-2}
\bbE[S_{a_4}] \overset{\eqref{eq:afl-s4-end}}&{\leq} 4Q\sum\limits_{i=1}^n\sum\limits_{q{=}0}^{Q{-}1}\bbE_{p_i}\Bigg[\Big\lVert\tilde{\nabla}f_i\left(w_{i,q}^{\Omega(t)}\right) - \nabla f_i\left(w_{i,q}^{\Omega(t)}\right)\Big\rVert^2\Bigg]\\
&+ 4Q\sum\limits_{i=1}^n\sum\limits_{q{=}0}^{Q{-}1}\bbE\Big\lVert\nabla f_i\left(w_{i,q}^{\Omega(t)}\right) - \nabla f_i\left(w^t\right)\Big\rVert^2\\
& + 4Q\sum\limits_{i=1}^n\sum\limits_{q{=}0}^{Q{-}1}\bbE\Big\lVert\nabla f_i\left(w^t\right) - \nabla f\left(w^t\right)\Big\rVert^2\\
&+ 4Q\sum\limits_{i=1}^n\sum\limits_{q{=}0}^{Q{-}1}\bbE\Big\lVert\nabla f\left(w^t\right) \Big\rVert^2\\
\overset{\eqref{eq:bounded-variance-batch}, \eqref{eq:bounded-heterogeneity}}&{\leq} 4nQ^2\left[\sigma_a^2 + \gamma_g^2 + \bbE\Big\lVert\nabla f\left(w^t\right)\Big\rVert^2 \right]\label{eq:afl-s4-2-end-1}\\
&+ 4Q\sum\limits_{i=1}^n\sum\limits_{q{=}0}^{Q{-}1}\bbE\left\lVert\nabla f_i\left(w_{i,q}^{\Omega(t)}\right)-\nabla f_i(w^{t})\right\rVert^2.\label{eq:afl-s4-2-end}
\end{align}
Therefore, due to \eqref{eq:afl-main-1}-\eqref{eq:afl-s3-end} and \eqref{eq:afl-s4-2-end-1}-\eqref{eq:afl-s4-2-end}, we have
\begin{align}\label{eq:afl-main-2}
    \bbE f\left(w^{t{+}1}\right) &\leq \bbE f(w^{t}) -\left[\frac{\eta\beta(2Q{-}1)}{2} - 2\eta^2L\beta^2Q^2\right]\bbE\left\lVert\nabla f(w^{t})\right\rVert^2\\
    &+ \left[\frac{\eta\beta Q}{2n}+\frac{2\eta^2\beta^2QL}{n}\right] \sum\limits_{i=1}^n\sum\limits_{q{=}0}^{Q{-}1}\bbE\left\lVert\nabla f_i\left(w_{i,q}^{\Omega(t)}\right)-\nabla f_i(w^{t})\right\rVert^2\\
    &+ 2\eta^2L\beta^2Q^2\sigma_a^2 + 2\eta^2L\beta^2Q^2\gamma_g^2\\
    \overset{\eqref{eq:smoothness}}&{\leq} \bbE f(w^{t}) -\left[\frac{\eta\beta(2Q{-}1)}{2} - 2\eta^2L\beta^2Q^2\right]\bbE\left\lVert\nabla f(w^{t})\right\rVert^2\\
    &+ \frac{\eta\beta Q L^2\left(1{+}4\eta\beta L\right)}{2n} \bbE\underbrace{\sum\limits_{i=1}^n\sum\limits_{q{=}0}^{Q{-}1}\left\lVert w_{i,q}^{\Omega(t)}-w^{t}\right\rVert^2}_{=: S_{a_5}}\\
    &+ 2\eta^2L\beta^2Q^2\sigma_a^2 + 2\eta^2L\beta^2Q^2\gamma_g^2.\label{eq:afl-main-2-end}
\end{align}
Thus, it is sufficient to bound the following expression in $S_{a_5}$:
\begin{align}\label{eq:afl-s5}
\Big\lVert w^{t}&-w_{i,q}^{\Omega(t)}\Big\rVert^2\\
&= \left\lVert \sum\limits_{s{=}\Omega(t)}^{t{-}1}\left(w^{s{+}1}-w^{s}\right)+w^{\Omega(t)}-w_{i,q}^{\Omega(t)}\right\rVert^2\\
\overset{\eqref{eq:gen-ineq-1}}&{\leq}\left(1{+}\frac{1}{\beta^2}\right)\left\lVert \sum\limits_{s{=}\Omega(t)}^{t{-}1}\left(w^{s{+}1}-w^{s}\right)\right\rVert^2+\left(1{+}\beta^2\right)\left\lVert w^{\Omega(t)}-w_{i,q}^{\Omega(t)}\right\rVert^2\\
\overset{\eqref{eq:gen-ineq-4}}&{\leq}\left(t{-}\Omega(t)\right)\left(1{+}\frac{1}{\beta^2}\right)\left[\sum\limits_{s{=}\Omega(t)}^{t{-}1}\left\lVert w^{s{+}1}-w^{s}\right\rVert^2\right]+\left(1{+}\beta^2\right)\left\lVert w^{\Omega(t)}-w_{i,q}^{\Omega(t)}\right\rVert^2\\
\overset{\eqref{eq:staleness}}&{\leq}\tau\left(1{+}\frac{1}{\beta^2}\right)\underbrace{\left[\sum\limits_{s{=}t{-}\tau}^{t{-}1}\left\lVert w^{s{+}1}-w^{s}\right\rVert^2\right]}_{=: S_{a_7}}+\left(1{+}\beta^2\right)\underbrace{\left\lVert w^{\Omega(t)}-w_{i,q}^{\Omega(t)}\right\rVert^2}_{=: S_{a_6}}.\label{eq:afl-s5-end}
\end{align}

Now, we show a bound on the evolution of local updates at an arbitrary round $s\geq 0$, i.e., the distance between $w_{i,q}^s$ and $w^s$:
\begin{align}\label{eq:afl-s6}
\bbE\left\lVert w_{i,q}^s - w^s\right\rVert^2 &= \bbE\left\lVert w_{i,q{-}1}^s - \eta\tilde{\nabla}f_i\left(w_{i,q{-}1}^s\right) - w^s\right\rVert^2\\
&= \bbE\Big\lVert w_{i,q{-}1}^s - w^s - \eta\nabla f\left(w^s\right)\nonumber\\ &\quad-\eta\tilde{\nabla}f_i\left(w_{i,q{-}1}^s\right) + \eta\nabla f_i\left(w_{i,q{-}1}^s\right)\nonumber\\
&\quad- \eta\nabla f_i\left(w_{i,q{-}1}^s\right) + \eta\nabla f_i\left(w^s\right)\nonumber\\
&\quad- \eta\nabla f_i\left(w^s\right) + \eta\nabla f\left(w^s\right) \Big\rVert^2\\
\overset{\eqref{eq:gen-ineq-1}}&{\leq} \left(1{+}\frac{1}{2Q}\right)\bbE\Big\lVert w_{i,q{-}1}^s - w^s\Big\rVert^2\\
&+\hspace{1em}4(1{+}2Q)\eta^2\bbE\Bigg[\Big\lVert\tilde{\nabla}f_i\left(w_{i,q{-}1}^s\right) - \nabla f_i\left(w_{i,q{-}1}^s\right)\Big\rVert^2\nonumber\\
&\hspace{7.1em}+ \Big\lVert\nabla f_i\left(w_{i,q{-}1}^s\right) - \nabla f_i\left(w^s\right)\Big\rVert^2\nonumber\\
&\hspace{7.1em}+ \Big\lVert\nabla f_i\left(w^s\right) - \nabla f\left(w^s\right) \Big\rVert^2\nonumber\\
&\hspace{7.1em}+ \Big\lVert\nabla f\left(w^s\right)\Big\rVert^2
\Bigg]\\
\overset{\eqref{eq:smoothness}, \eqref{eq:bounded-variance-batch}}&{\leq} \left(1{+}\frac{1}{2Q}\right)\bbE\Big\lVert w_{i,q{-}1}^s - w^s\Big\rVert^2\\
&+\hspace{1em}4(1{+}2Q)\eta^2\Bigg[\sigma_a^2 + L^2\,\bbE\Big\lVert w_{i,q{-}1}^s - w^s\Big\rVert^2\nonumber\\
&\hspace{7.1em}+ \bbE\Big\lVert\nabla f_i\left(w^s\right) - \nabla f\left(w^s\right) \Big\rVert^2 + \bbE\Big\lVert\nabla f\left(w^s\right)\Big\rVert^2
\Bigg].\label{eq:afl-s6-end}
\end{align}
Note that we can select stepsize $\eta\leq\frac{1}{4L(Q{+}1)}$ such that
\begin{align}\label{eq:stepsize-1}
    \eta^2 \leq \frac{1}{16L^2(Q{+}1)^2} \leq \frac{1}{8L^2Q(2Q{+}1)} \Rightarrow 4(1{+}2Q)\eta^2L^2 \leq \frac{1}{2Q},
\end{align}
therefore, due to \eqref{eq:afl-s6}-\eqref{eq:stepsize-1}, we have:
\begin{align}\label{eq:afl-s6-2}
&\underbrace{\bbE\left\lVert w_{i,q}^s - w^s\right\rVert^2}_{:= P_{i,q}^s} \leq \underbrace{\left(1{+}\frac{1}{Q}\right)\bbE\Big\lVert w_{i,q{-}1}^s - w^s\Big\rVert^2}_{:= P_{i,q{-}1}^s}\\
&\qquad\quad+\underbrace{4(1{+}2Q)\eta^2\Bigg[\sigma_a^2 + \bbE\Big\lVert\nabla f_i\left(w^s\right) - \nabla f\left(w^s\right) \Big\rVert^2 + \bbE\Big\lVert\nabla f\left(w^s\right)\Big\rVert^2
\Bigg]}_{:= R_i^s} \Rightarrow\\
P_{i,q}^s &\leq \left(1{+}\frac{1}{Q}\right) P_{i,q{-}1}^s + R_i^s\\
&= R_i^s\sum\limits_{k=0}^{q{-}1} \left(1{+}\frac{1}{Q}\right)^k \leq R_i^s\sum\limits_{k=0}^{Q{-}1} \left(1{+}\frac{1}{Q}\right)^k\\
&= R_i^s \frac{\left(1{+}\frac{1}{Q}\right)^Q-1}{\left(1{+}\frac{1}{Q}\right)-1} = R_i^s Q\left[\left(1{+}\frac{1}{Q}\right)^Q-1\right]\leq R_i^s Q(e - 1) \leq 2R_i^s Q \Rightarrow\\
\bbE&\left\lVert w_{i,q}^s - w^s\right\rVert^2 \leq 8Q(1{+}2Q)\eta^2\Bigg[\sigma_a^2 + \bbE\Big\lVert\nabla f_i\left(w^s\right) - \nabla f\left(w^s\right) \Big\rVert^2 + \bbE\Big\lVert\nabla f\left(w^s\right)\Big\rVert^2
\Bigg],\label{eq:afl-s6-2-end}
\end{align}
for all $q\in[Q]$ and $s\geq 0$. We now will use \eqref{eq:afl-s6}-\eqref{eq:afl-s6-2-end} to provide a bound on the expression in $S_{a_7}$. Again, note that according to Algorithms \ref{alg:server} \& \ref{alg:client}, we have:
\begin{align}\label{eq:afl-s7}
w^{s{+}1} &= w^s - \beta\left(w_{i_s,0}^{\Omega(s)} - w_{i_s,Q}^{\Omega(s)}\right)\Rightarrow\\
\bbE\left\lVert w^{s{+}1} - w^s\right\rVert^2 &\leq \beta^2\,\bbE\left\lVert w_{i_s,Q}^{\Omega(s)} - w^{\Omega(s)}\right\rVert^2\\
&= \beta^2\,\bbE\left[\bbE_{i_s} \left\lVert w_{i_s,Q}^{\Omega(s)} - w^{\Omega(s)}\right\rVert^2\right]\\
&= \frac{\beta^2}{n}\sum\limits_{j{=}1}^n\bbE\left\lVert w_{j,Q}^{\Omega(s)} - w^{\Omega(s)}\right\rVert^2\\
&\leq 8Q(1{+}2Q)\eta^2\beta^2\Bigg[\sigma_a^2 + \gamma_g^2 + \bbE\Big\lVert\nabla f\left(w^{\Omega(s)}\right)\Big\rVert^2
\Bigg].\label{eq:afl-s7-end}
\end{align}
Let $\phi=8\eta^2Q^2(1{+}2Q)(1{+}\beta^2)$, then according to \eqref{eq:afl-s5}-\eqref{eq:afl-s7-end}
\begin{align}\label{eq:afl-s5-2}
\frac{1}{n\phi}\bbE[S_{a_5}]&\leq\tau\left[\sum\limits_{s{=}t{-}\tau}^{t{-}1}\left\lVert w^{s{+}1}-w^{s}\right\rVert^2\right]+\frac{1}{nQ}\sum_{i=1}^{n}\sum_{q=0}^{Q{-}1}\left\lVert w^{\Omega(t)}-w_{i,q}^{\Omega(t)}\right\rVert^2.\\
&\leq \tau^2 \sigma_a^2 + \tau^2 \gamma_g^2 + \tau\sum_{s=t{-}\tau}^{t{-}1} \bbE\Big\lVert\nabla f\left(w^{\Omega(s)}\right)\Big\rVert^2\\
& + \sigma_a^2 + \gamma_g^2 + \bbE\Big\lVert\nabla f\left(w^{\Omega(t)}\right)\Big\rVert^2\\
&\leq (\tau^2{+}1) \left[\sigma_a^2 + \gamma_g^2\right] + \bbE\Big\lVert\nabla f\left(w^{\Omega(t)}\right)\Big\rVert^2 + \tau \sum_{s=t{-}\tau}^{t{-}1} \bbE\Big\lVert\nabla f\left(w^{\Omega(s)}\right)\Big\rVert^2\\
&\leq (\tau^2{+}1) \left[\sigma_a^2 + \gamma_g^2\right] + \tau \sum_{s=t{-}\tau}^{t} \bbE\Big\lVert\nabla f\left(w^{\Omega(s)}\right)\Big\rVert^2.
\label{eq:afl-s5-2-end}
\end{align}
Thus, by combining \eqref{eq:afl-main-2}-\eqref{eq:afl-s5-2-end}, we have the following inequality:
\begin{align}\label{eq:afl-main-3}
    \bbE f\left(w^{t{+}1}\right) &\leq \bbE f(w^{t}) -\eta\beta\left[\frac{2Q{-}1}{2} - 2\eta\beta LQ^2\right]\bbE\left\lVert\nabla f(w^{t})\right\rVert^2\\
    &+4\eta^3\beta L^2Q^3 (1{+}2Q)(1{+}\beta^2)(1{+}4\eta\beta L)\,\tau\left[\sum\limits_{s{=}t{-}\tau}^{t}\bbE\left\lVert\nabla f\left(w^{\Omega(s)}\right)\right\rVert^2\right]\\
    &+ 4\eta^3\beta L^2Q^3 (1{+}2Q) (\tau^2{+}1)(1{+}\beta^2)(1{+}4\eta\beta L)
    \left(\sigma_a^2 +\gamma_g^2\right)\\
    &+ 2\eta^2\beta^2 LQ^2 \left(\sigma_a^2 +\gamma_g^2\right), \label{eq:afl-main-3-end}
\end{align}
where by rearranging, we obtain the following inequality:
\begin{align}\label{eq:afl-main-4}
    &\left(1 - 4\eta\beta LQ\right)\bbE\left\lVert\nabla f(w^{t})\right\rVert^2\\
    & \qquad -8\eta^2 L^2Q^2 (1{+}2Q)(1{+}\beta^2)(1{+}4\eta\beta L)\,\tau\left[\sum\limits_{s{=}t{-}\tau}^{t}\bbE\left\lVert\nabla f\left(w^{\Omega(s)}\right)\right\rVert^2\right]\\
    &\leq \frac{2\left[\bbE f(w^{t})-\bbE f\left(w^{t{+}1}\right)\right]}{\eta\beta Q}\\
    &\qquad +8\eta^2 L^2Q^2 (1{+}2Q) (\tau^2{+}1)(1{+}\beta^2)(1{+}4\eta\beta L)
    \left(\sigma_a^2 +\gamma_g^2\right)\\
    &\qquad +4\eta\beta L Q \left(\sigma_a^2 +\gamma_g^2\right). \label{eq:afl-main-4-end}
\end{align}
Now, note that for any $s\geq 0$,\footnote{For $s<\tau$, the right-hand side of the inequality consists of fewer terms.}
\begin{align}\label{eq:afl-s5-3}
\bbE\Big\lVert\nabla f\left(w^{\Omega(s)}\right)\Big\rVert^2 \leq \sum_{u=s-\tau}^{s} \bbE\Big\lVert\nabla f\left(w^{u}\right)\Big\rVert^2,
\end{align}
Therefore, we add up the inequality in \eqref{eq:afl-main-4}-\eqref{eq:afl-main-4-end}, for $t=0,1,\dots T{-}1$, and obtain
\begin{align}\label{eq:afl-main-5}
    \big[1 -4\eta\beta LQ -8\eta^2 L^2&Q^2 (1{+}2Q)\tau(\tau{+}1)^2(1{+}\beta^2)(1{+}4\eta\beta L) \big]\frac{\sum\limits_{t=0}^{T{-}1}\bbE\left\lVert\nabla f(w^{t})\right\rVert^2}{T}\\
    \leq \frac{2\left[f(w^{0})-\bbE f(w^{T})\right]}{\eta\beta QT} &+ 4\eta\beta L Q \left(\sigma_a^2 +\gamma_g^2\right)\\
    &+ 8\eta^2 L^2Q^2 (1{+}2Q) (\tau^2{+}1)(1{+}\beta^2)(1{+}4\eta\beta L)
    \left(\sigma_a^2 +\gamma_g^2\right). \label{eq:afl-main-5-end}
\end{align}
Thus, by setting $\beta=1$ and $\eta=\frac{1}{Q\sqrt{LT}}$, we can simply see that
\begin{align}
1 -4\eta\beta LQ -8\eta^2 L^2&Q^2 (1{+}2Q)\tau(\tau{+}1)^2(1{+}\beta^2)(1{+}4\eta\beta L) \geq \frac{1}{2},\\
\eta \leq \frac{1}{4L(Q{+}1)},
\end{align}
for $T\geq 160L(Q{+}7)(\tau{+}1)^3$. Therefore, we can conclude the final result in Theorem \ref{thm:afl} under this choice of $\eta$ and $\beta$.
\end{proof}

\clearpage

\section{Personalized Asynchronous Federated Learning: \mtext{MAML}}\label{app:persafl-maml}

\begin{proof}[Proof of Theorem \ref{thm:persafl-maml}]
To simplify \eqref{eq:persafl-maml-stoch-grad}, we denote $\tilde{\nabla}F^{(b)}_{i}\left(w\right) = \nabla \tilde{F}^{(b)}_i\left(w, \mcD_i'', \mcD_i', \mcD_i\right)$. Then similar to \eqref{eq:afl-update-rule-delay}, at round $t$, the update rule for \textcolor{brickred}{Option B} can be written as follows:
\begin{align}\label{eq:persafl-maml-update-rule-delay}
    w^{t{+}1} = w^{t} - \eta\beta \sum\limits_{q=0}^{Q{-}1} \tilde{\nabla} F^{(b)}_{i_t}\left(w_{i_t,q}^{\Omega(t)}\right).
\end{align}
According to Lemma \ref{lem:smoothness-maml},
\begin{align}\label{eq:persafl-maml-main-0}
    F^{(b)}\left(w^{t{+}1}\right) \overset{\eqref{eq:smoothness}}{\leq} 
    F^{(b)}(w^{t}) &-\eta\beta\underbrace{\left\langle \nabla F^{(b)}(w^{t}),  \sum\limits_{q=0}^{Q{-}1} \tilde{\nabla} F^{(b)}_{i_t}\left(w_{i_t,q}^{\Omega(t)}\right) \right\rangle}_{=: S_{b_1}}\nonumber\\
    &+\frac{L_b\eta^2\beta^2}{2}\underbrace{\left\lVert \sum\limits_{q=0}^{Q{-}1} \tilde{\nabla} F^{(b)}_{i_t}\left(w_{i_t,q}^{\Omega(t)}\right)\right\rVert^2}_{=: S_{b_2}}
\end{align}
Similar to the inequalities in \eqref{eq:afl-s1}-\eqref{eq:afl-s1-end}, we first show a lower bound on term $S_{b_1}$ in \eqref{eq:persafl-maml-main-0}. We also denote \mbox{$\tilde{g}_i^t=\sum\limits_{q{=}0}^{Q{-}1} \tilde{\nabla}F^{(b)}_i\left(w_{i,q}^{\Omega(t)}\right)$}, \mbox{$\tilde{g}^t=\frac{1}{n}\sum\limits_{i{=}1}^n\tilde{g}_i^t$}, \mbox{$g_i^t=\sum\limits_{q{=}0}^{Q{-}1}\nabla F^{(b)}_i\left(w_{i,q}^{\Omega(t)}\right)$}, and \mbox{$g^t=\frac{1}{n}\sum\limits_{i{=}1}^n g_i^t$} for simplicity. Note that $\tilde{g}_i^t$ and $g_i^t$ are the stochastic and deterministic gradients of the personalized cost functions $F^{(b)}_i$ at stale parameters. According to these definitions, we have
\begin{align}\label{eq:persafl-maml-unbiased-acc}
    \left\lVert\bbE\left[\tilde{g}^t-g^t\right]\right\rVert
    \overset{\eqref{eq:gen-ineq-2}}&{\leq} \frac{1}{n}\sum_{i=1}^{n} \left\lVert\bbE\left[\tilde{g}_i^t-g_i^t\right]\right\rVert\\
    \overset{\eqref{eq:gen-ineq-2}}&{\leq} \frac{1}{n}\sum_{i=1}^{n}\sum_{q=0}^{Q{-}1} \left\lVert\bbE\left[\tilde{\nabla}F^{(b)}_i\left(w_{i,q}^{\Omega(t)}\right)-\nabla F^{(b)}_i\left(w_{i,q}^{\Omega(t)}\right)\right]\right\rVert\\
    \overset{\eqref{eq:unbiased-mean-maml}}&{\leq} \frac{1}{n}\sum_{i=1}^{n}\sum_{q=0}^{Q{-}1} \mu_b = Q\mu_b, \label{eq:persafl-maml-unbiased-acc-end}
\end{align}
where as we discussed in \eqref{eq:unbiased-mean-maml}, $\mu_b$ measures the unbiasedness in the estimation of the personalized stochastic gradient.

\begin{align}\label{eq:persafl-maml-s1}
\bbE\left[S_{b_1}\right] &= \bbE\left[\bbE_{i_t}\left\langle\nabla F^{(b)}(w^{t}), \tilde{g}_{i_t}^t\right\rangle\right]\\
&= \bbE\left[\left\langle \nabla F^{(b)}(w^{t}),  \frac{1}{n}\sum\limits_{i=1}^{n}\tilde{g}_{i}^t \right\rangle\right] = \bbE\left[\left\langle \nabla F^{(b)}(w^{t}), \tilde{g}^t\right\rangle\right]\\
&= Q\,\bbE\left\lVert\nabla F^{(b)}(w^{t})\right\rVert^2 + \bbE\left\langle \nabla F^{(b)}(w^{t}),  \bbE\left[\tilde{g}^t-g^t\right]\right\rangle\\
&+ \bbE\left\langle \nabla F^{(b)}(w^{t}),  g^t-Q\nabla F^{(b)}(w^{t})\right\rangle\\
\overset{\eqref{eq:gen-ineq-3}}&{\geq} Q\,\bbE\left\lVert\nabla F^{(b)}(w^{t})\right\rVert^2 - \frac{1}{4}\bbE\left\lVert \nabla F^{(b)}(w^{t}) \right\rVert^2 - \left\lVert \bbE\left[g^t-\tilde{g}^t\right]\right\rVert^2\\
& - \frac{1}{4}\bbE\left\lVert \nabla F^{(b)}(w^{t}) \right\rVert^2 - \bbE\left\lVert g^t-Q\nabla F^{(b)}(w^{t})\right\rVert^2\\
\overset{\eqref{eq:persafl-maml-unbiased-acc}-\eqref{eq:persafl-maml-unbiased-acc-end}}&{\geq} \frac{2Q{-}1}{2}\bbE\left\lVert\nabla F^{(b)}(w^{t})\right\rVert^2 - \bbE\left\lVert g^t-Q\nabla F^{(b)}(w^{t})\right\rVert^2-Q^2\mu_b^2,\label{eq:persafl-maml-s1-end}
\end{align}
and
\begin{align}\label{eq:persafl-maml-s2}
\bbE_{i_t}\left[S_{b_2}\right] = \bbE_{i_t}\left\lVert\sum\limits_{q=0}^{Q{-}1} \tilde{\nabla} F^{(b)}_{i_t}\left(w_{i_t,q}^{\Omega(t)}\right)\right\rVert^2 = \frac{1}{n} \sum\limits_{i=1}^n \left\lVert \tilde{g}_i^t\right\rVert^2.
\end{align}
Therefore, according to \eqref{eq:persafl-maml-main-0}, \eqref{eq:persafl-maml-s1-end}, and \eqref{eq:persafl-maml-s2},
\begin{align}\label{eq:persafl-maml-main-1}
    \bbE F^{(b)}\left(w^{t{+}1}\right) &\leq \bbE F^{(b)}(w^{t}) -\frac{\eta\beta(2Q{-}1)}{2}\bbE\left\lVert\nabla F^{(b)}(w^{t})\right\rVert^2 + \eta\beta Q^2\mu_b^2\\
    &+ \eta\beta\,\bbE\underbrace{\left\lVert g^t-Q\nabla F^{(b)}(w^{t})\right\rVert^2}_{=: S_{b_3}} +\frac{L_b\eta^2\beta^2}{2n}\bbE\underbrace{\sum\limits_{i=1}^n \left\lVert \tilde{g}_i^t\right\rVert^2}_{=: S_{b_4}},\label{eq:persafl-maml-main-1-end}
\end{align}
where similar to \eqref{eq:afl-s3}-\eqref{eq:afl-s3-end}, we can bound $S_{b_3}$ as follows:
\begin{align}\label{eq:persafl-maml-s3}
S_{b_3} \leq \frac{Q}{n}\sum\limits_{i=1}^n\sum\limits_{q=0}^{Q{-}1} \left\lVert\nabla F^{(b)}_i\left(w_{i,q}^{\Omega(t)}\right)-\nabla F^{(b)}_i(w^{t})\right\rVert^2.
\end{align}
Moreover, we can show an upper bound on $S_{b_4}$ akin to \eqref{eq:persafl-maml-s4}-\eqref{eq:persafl-maml-s4-end}:
\begin{alignat}{2}\label{eq:persafl-maml-s4}
S_{b_4} &= \,\,\sum\limits_{i=1}^n \Big\lVert \sum\limits_{q{=}0}^{Q{-}1}\tilde{\nabla}&&F^{(b)}_i\left(w_{i,q}^{\Omega(t)}\right)\Big\rVert^2\\
\overset{\eqref{eq:gen-ineq-4}}&{\leq} Q\sum\limits_{i=1}^n\sum\limits_{q{=}0}^{Q{-}1} &&\Big\lVert\tilde{\nabla}F^{(b)}_i\left(w_{i,q}^{\Omega(t)}\right)\Big\rVert^2\\
&= \,\,\,Q\sum\limits_{i=1}^n\sum\limits_{q{=}0}^{Q{-}1} && \Big\lVert\tilde{\nabla}F^{(b)}_i\left(w_{i,q}^{\Omega(t)}\right) - \nabla F^{(b)}_i\left(w_{i,q}^{\Omega(t)}\right) + \nabla F^{(b)}_i\left(w_{i,q}^{\Omega(t)}\right) - \nabla F^{(b)}_i\left(w^t\right)\nonumber\\
&&& + \nabla F^{(b)}_i\left(w^t\right) - \nabla F^{(b)}\left(w^t\right) + \nabla F^{(b)}\left(w^t\right) \Big\rVert^2\\
\overset{\eqref{eq:gen-ineq-4}}&{\leq} 4Q\sum\limits_{i=1}^n\sum\limits_{q{=}0}^{Q{-}1}&&\Bigg[\Big\lVert\tilde{\nabla}F^{(b)}_i\left(w_{i,q}^{\Omega(t)}\right) - \nabla F^{(b)}_i\left(w_{i,q}^{\Omega(t)}\right)\Big\rVert^2\nonumber\\
&&&+ \Big\lVert\nabla F^{(b)}_i\left(w_{i,q}^{\Omega(t)}\right) - \nabla F^{(b)}_i\left(w^t\right)\Big\rVert^2\nonumber\\
&&& + \Big\lVert\nabla F^{(b)}_i\left(w^t\right) - \nabla F^{(b)}\left(w^t\right)\Big\rVert^2\nonumber\\
&&&+ \Big\lVert\nabla F^{(b)}\left(w^t\right) \Big\rVert^2 \Bigg]\label{eq:persafl-maml-s4-end} \Rightarrow
\end{alignat}

\begin{align}\label{eq:persafl-maml-s4-2}
\bbE[S_{b_4}] \overset{\eqref{eq:persafl-maml-s4-end}}&{\leq} 4Q\sum\limits_{i=1}^n\sum\limits_{q{=}0}^{Q{-}1}\bbE_{p_i}\Bigg[\Big\lVert\tilde{\nabla}F^{(b)}_i\left(w_{i,q}^{\Omega(t)}\right) - \nabla F^{(b)}_i\left(w_{i,q}^{\Omega(t)}\right)\Big\rVert^2\Bigg]\\
&+ 4Q\sum\limits_{i=1}^n\sum\limits_{q{=}0}^{Q{-}1}\bbE\Big\lVert\nabla F^{(b)}_i\left(w_{i,q}^{\Omega(t)}\right) - \nabla F^{(b)}_i\left(w^t\right)\Big\rVert^2\\
& + 4Q\sum\limits_{i=1}^n\sum\limits_{q{=}0}^{Q{-}1}\bbE\Big\lVert\nabla F^{(b)}_i\left(w^t\right) - \nabla F^{(b)}\left(w^t\right)\Big\rVert^2\\
&+ 4Q\sum\limits_{i=1}^n\sum\limits_{q{=}0}^{Q{-}1}\bbE\Big\lVert\nabla F^{(b)}\left(w^t\right) \Big\rVert^2\\
\overset{\eqref{eq:bounded-variance-batch}, \eqref{eq:bounded-heterogeneity}}&{\leq} 4nQ^2\left[\sigma_b^2 + \gamma_b^2 + \bbE\Big\lVert\nabla F^{(b)}\left(w^t\right)\Big\rVert^2 \right]\label{eq:persafl-maml-s4-2-end-1}\\
&+ 4Q\sum\limits_{i=1}^n\sum\limits_{q{=}0}^{Q{-}1}\bbE\left\lVert\nabla F^{(b)}_i\left(w_{i,q}^{\Omega(t)}\right)-\nabla F^{(b)}_i(w^{t})\right\rVert^2.\label{eq:persafl-maml-s4-2-end}
\end{align}
Therefore, due to \eqref{eq:persafl-maml-main-1}-\eqref{eq:persafl-maml-s3} and \eqref{eq:persafl-maml-s4-2-end-1}-\eqref{eq:persafl-maml-s4-2-end}, we have
\begin{align}\label{eq:persafl-maml-main-2}
    \bbE F^{(b)}\left(w^{t{+}1}\right) &\leq \bbE F^{(b)}(w^{t}) -\left[\frac{\eta\beta(2Q{-}1)}{2} - 2\eta^2L_b\beta^2Q^2\right]\bbE\left\lVert\nabla F^{(b)}(w^{t})\right\rVert^2\\
    &+ \left[\frac{\eta\beta Q}{n}+\frac{2\eta^2\beta^2 Q L_b}{n}\right] \sum\limits_{i=1}^n\sum\limits_{q{=}0}^{Q{-}1}\bbE\left\lVert\nabla F^{(b)}_i\left(w_{i,q}^{\Omega(t)}\right)-\nabla F^{(b)}_i(w^{t})\right\rVert^2\\
    &+ \eta\beta Q^2\mu_b^2 + 2\eta^2L_b\beta^2Q^2\sigma_b^2 + 2\eta^2L_b\beta^2Q^2\gamma_b^2\\
    \overset{\eqref{eq:smoothness}}&{\leq} \bbE F^{(b)}(w^{t}) -\left[\frac{\eta\beta(2Q{-}1)}{2} - 2\eta^2L_b\beta^2Q^2\right]\bbE\left\lVert\nabla F^{(b)}(w^{t})\right\rVert^2\label{eq:persafl-maml-main-2-s5}\\
    &+ \frac{\eta\beta Q L_b^2\left(1{+}2\eta\beta L_b\right)}{n} \sum\limits_{i=1}^n\sum\limits_{q{=}0}^{Q{-}1}\bbE\underbrace{\left\lVert w_{i,q}^{\Omega(t)}-w^{t}\right\rVert^2}_{=: S_{b_5}}\\
    &+ \eta\beta Q^2\mu_b^2 + 2\eta^2L_b\beta^2Q^2\left(\sigma_b^2 +\gamma_b^2\right).\label{eq:persafl-maml-main-2-end}
\end{align}
Now, we provide an upper bound on $S_{b_5}$ in \eqref{eq:persafl-maml-main-2-s5} as follows:
\begin{align}\label{eq:persafl-maml-s5}
S_{b_5} = \left\lVert w^{t}-w_{i,q}^{\Omega(t)}\right\rVert^2 &= \left\lVert w^{t} - w^{\Omega(t)} + w^{\Omega(t)} - w_{i,q}^{\Omega(t)} \right\rVert^2\\
\overset{\eqref{eq:gen-ineq-1}}&{\leq} 2\underbrace{\left\lVert w^{\Omega(t)} - w_{i,q}^{\Omega(t)} \right\rVert^2}_{S_{b_6}} + 2\underbrace{\left\lVert w^{t} - w^{\Omega(t)} \right\rVert^2}_{S_{b_7}},\label{eq:persafl-maml-s5-end}
\end{align}
where the first term determines the evolution of local updates and the second term considers the effect of asynchronous updates. Therefore, using Lemma \ref{lem:bounded-gradient-maml}, we have
\begin{align}\label{eq:persafl-maml-s6}
\bbE[S_{b_6}] &= \bbE\left\lVert w^{\Omega(t)} - w_{i,q}^{\Omega(t)}\right\rVert^2\\
&= \bbE\left\lVert w_{i,0}^{\Omega(t)} - w_{i,q}^{\Omega(t)}\right\rVert^2\\
\overset{\eqref{eq:persafl-maml-update-rule-delay}}&{=} \eta^2 \bbE\left\lVert \sum_{r{=}0}^{q{-}1} \tilde{\nabla}F^{(b)}_i\left(w_{i,r}^{\Omega(t)}\right) \right\rVert^2\\
\overset{\eqref{eq:gen-ineq-4}}&{\leq} \eta^2q \sum_{r{=}0}^{q{-}1} \bbE\left\lVert\tilde{\nabla}F^{(b)}_i\left(w_{i,r}^{\Omega(t)}\right)\right\rVert^2\\
\overset{\eqref{eq:gen-ineq-4}}&{\leq} 2\eta^2q \sum_{r{=}0}^{q{-}1} \left[ \bbE\left\lVert\tilde{\nabla}F^{(b)}_i\left(w_{i,r}^{\Omega(t)}\right) - \nabla F^{(b)}_i\left(w_{i,r}^{\Omega(t)}\right)\right\rVert^2 + \bbE\left\lVert\nabla F^{(b)}_i\left(w_{i,r}^{\Omega(t)}\right)\right\rVert^2\right]\\
\overset{\eqref{eq:bounded-variance-maml},\eqref{eq:bounded-gradient-maml}}&{\leq} 2\eta^2q \sum_{r{=}0}^{q{-}1} \left(G_b^2 + \sigma_b^2\right) = 2\eta^2q^2\left(G_b^2 + \sigma_b^2\right),\label{eq:persafl-maml-s6-end}
\end{align}

\begin{align}\label{eq:persafl-maml-s7}
\bbE[S_{b_7}] &= \bbE\left\lVert w^{t} -  w^{\Omega(t)}\right\rVert^2\\
&= \bbE\left\lVert \sum_{s{=}\Omega(t)}^{t{-}1} \left(w^{s{+}1} - w^{s}\right)\right\rVert^2\\
\overset{\text{Alg. } \ref{alg:server}, \ref{alg:client}}&{=} \eta^2\beta^2 \bbE\left\lVert \sum_{s{=}\Omega(t)}^{t{-}1}\sum_{q{=}0}^{Q{-}1} \tilde{\nabla}F^{(b)}_{i_s}\left(w_{{i_s},q}^{\Omega(s)}\right) \right\rVert^2\\
\overset{\eqref{eq:gen-ineq-4}}&{\leq} \eta^2\beta^2Q\, (t{-}\Omega(t)) \sum_{s{=}\Omega(t)}^{t{-}1}\sum_{q{=}0}^{Q{-}1} \bbE\left\lVert \tilde{\nabla}F^{(b)}_{i_s}\left(w_{{i_s},q}^{\Omega(s)}\right) \right\rVert^2\\
\overset{\eqref{eq:staleness}}&{\leq} 2\eta^2\beta^2Q\,\tau \sum_{s{=}t{-}\tau}^{t{-}1}\sum_{q{=}0}^{Q{-}1} \Bigg[ \bbE\left\lVert\tilde{\nabla}F^{(b)}_{i_s}\left(w_{{i_s},q}^{\Omega(s)}\right)- \nabla F^{(b)}_{i_s}\left(w_{{i_s},q}^{\Omega(s)}\right)\right\rVert^2\nonumber\\
& \qquad\qquad\qquad\qquad\quad\,\,\,+ \bbE\left\lVert\nabla F^{(b)}_{i_s}\left(w_{{i_s},q}^{\Omega(s)}\right) \right\rVert^2\Bigg]\\
\overset{\eqref{eq:bounded-variance-maml},\eqref{eq:bounded-gradient-maml}}&{\leq} 2\eta^2\beta^2Q\tau^2 \sum_{q{=}0}^{Q{-}1} \left(G_b^2 + \sigma_b^2\right) = 2\eta^2\beta^2Q^2\tau^2\left(G_b^2 + \sigma_b^2\right).\label{eq:persafl-maml-s7-end}
\end{align}
So, according to \eqref{eq:persafl-maml-main-2}-\eqref{eq:persafl-maml-s7-end},
\begin{align}\label{eq:persafl-maml-main-3}
    \bbE F^{(b)}\left(w^{t{+}1}\right) &\leq \bbE F^{(b)}(w^{t}) -\frac{\eta\beta}{2}\left(2Q{-}1 - 4\eta\beta L_bQ^2\right)\bbE\left\lVert\nabla F^{(b)}(w^{t})\right\rVert^2\\
    &+ 4\eta^3\beta Q^4 L_b^2 \left(1{+}2\eta\beta L_b Q\right)\left(G_b^2{+}\sigma_b^2\right)\left(\beta^2\tau^2{+}1\right)\\
    &+ \eta\beta Q^2\mu_b^2 + 2\eta^2\beta^2L_bQ^2\sigma_b^2 + 2\eta^2\beta^2L_bQ^2\gamma_b^2, \label{eq:persafl-maml-main-3-end}
\end{align}
where by adding the terms in \eqref{eq:persafl-maml-main-3}-\eqref{eq:persafl-maml-main-3-end}, for $t=0,1,\dots T{-}1$, and rearranging them, we obtain the following inequality:
\begin{align}\label{eq:persafl-maml-main-4}
    \frac{1 - 4\eta\beta L_b Q}{T}\sum\limits_{t=0}^{T{-}1}\bbE\left\lVert\nabla F^{(b)}(w^{t})\right\rVert^2 &\leq \frac{2\left(F^{(b)}(w^{0})-\bbE F^{(b)}(w^{T})\right)}{\eta\beta QT}\\ 
    &+ 8\eta^2 Q^3 L_b^2 \left(1{+}2\eta\beta L_b Q\right)\left(G_b^2 {+} \sigma_b^2\right)\left(\beta^2\tau^2{+}1\right)\\
    &+ 4\eta\beta L_b Q\left(\sigma_b^2 + \gamma_b^2\right)\\
    &+ 2Q\mu_b^2.
    \label{eq:persafl-maml-main-4-end}
\end{align}
Finally, we can conclude the proof by fixing $\beta=1$ and $\eta\coloneqq \frac{1}{Q\sqrt{L_b T}}$ for $T\geq 64 L_b$, hence $\eta \leq \frac{1}{8\beta L_b Q}$.
\end{proof}


\clearpage
\section{Personalized Asynchronous Federated Learning: \mtext{ME}}\label{app:persafl-me}

We start by showing \eqref{eq:persafl-me-stoch-grad}. According to the definitions in \eqref{eq:persafl-me} and \eqref{eq:persafl-me-argmin}, we have
\begin{align}
    \hat{\theta}_i(w)&= \argmin_{\theta_i \in \bbR^d} \left[f_i(\theta_i) + \frac{\lambda}{2} \norm{\theta_i - w}^2 \right] \Rightarrow \nabla f_i\left(\hat{\theta}_i(w)\right) + \lambda\left[\hat{\theta}_i(w) - w\right] = 0,\label{eq:persafl-me-stationary}\\
    F^{(c)}_i(w) &= f_i\left(\hat{\theta}_i(w)\right) + \frac{\lambda}{2} \left\lVert\hat{\theta}_i(w)-w\right\rVert^2,\label{eq:persafl-me-stationary-end}
\end{align}
therefore,
\begin{align}\label{eq:persafl-me-exact-gradient}
    \nabla F^{(c)}_i(w) \overset{\eqref{eq:persafl-me-stationary-end}}&{=} \frac{\partial\,\hat{\theta}_i(w)}{\partial w} \left[\nabla f_i\left(\hat{\theta}_i(w)\right)\right] + \lambda \left[\frac{\partial\,\hat{\theta}_i(w)}{\partial w}-I\right]\left[\hat{\theta}_i(w)-w\right]\\
    \overset{\eqref{eq:persafl-me-stationary}}&{=} \lambda \frac{\partial\,\hat{\theta}_i(w)}{\partial w}\left[w-\hat{\theta}_i(w)\right] + \lambda \left[\frac{\partial\,\hat{\theta}_i(w)}{\partial w}-I\right]\left[\hat{\theta}_i(w)-w\right]\\
    & = \lambda \left[w-\hat{\theta}_i(w)\right].\label{eq:persafl-me-exact-gradient-end}
\end{align}


Before, presenting the proof of Theorem \ref{thm:persafl-me}, we proceed by providing the proof of Lemmas \ref{lem:smoothness-me}, \ref{lem:bounded-variance-me}, and \ref{lem:bounded-heterogeneity-me}.

\begin{proof}[Proof of Lemma \ref{lem:smoothness-me}]
Let $w,v$ be two arbitrary vectors in $\bbR^d$. Then, we have:
\begin{align}
    \nabla F_i^{(c)}(w) - \nabla F_i^{(c)}(y) \overset{\eqref{eq:persafl-me-exact-gradient-end}}&{=} \lambda \left[w-\hat{\theta}_i(w)\right] - \lambda\left[v-\hat{\theta}_i(v)\right]\\ \overset{\eqref{eq:persafl-me-stationary}}&{=} \nabla f_i\left(\hat{\theta}_i(w)\right) - \nabla f_i\left(\hat{\theta}_i(v)\right)\Rightarrow\label{eq:persafl-me-lem-smooth-1}\\
    \left\lVert\nabla F_i^{(c)}(w) - \nabla F_i^{(c)}(y)\right\rVert &= \left\lVert\nabla f_i\left(\hat{\theta}_i(w)\right) - \nabla f_i\left(\hat{\theta}_i(v)\right)\right\rVert\\
    \overset{\eqref{eq:smoothness}}&{\leq} L\left\lVert \hat{\theta}_i(w) - \hat{\theta}_i(v) \right\rVert\\
    \overset{\eqref{eq:persafl-me-stationary}}&{=} L\left\lVert w - \frac{1}{\lambda}\nabla f_i\left(\hat{\theta}_i(w)\right) - v + \frac{1}{\lambda}\nabla f_i\left(\hat{\theta}_i(v)\right) \right\rVert\\
    &\leq L\left\lVert w-v \right\rVert + \frac{L}{\lambda}\left\lVert \nabla f_i\left(\hat{\theta}_i(w)\right) - \nabla f_i\left(\hat{\theta}_i(v)\right) \right\rVert\\
    &= L\left\lVert w-v \right\rVert + \frac{L}{\lambda}\left\lVert \nabla F_i^{(c)}(w) - \nabla F_i^{(c)}(y) \right\rVert \Rightarrow\\
    \left\lVert\nabla F_i^{(c)}(w) - \nabla F_i^{(c)}(y)\right\rVert &\leq \frac{\lambda L}{\lambda-L}\left\lVert w-v \right\rVert,
\end{align}
which means $F^{(c)}_i$ is $\frac{\lambda L}{\lambda-L}$-smooth. Note that for $\lambda\geq \kappa L$, for some $\kappa>1$,
\begin{align}
    \frac{\lambda L}{\lambda-L} \leq L_c\coloneqq\frac{\lambda}{\kappa-1}
\end{align}
This concludes the statement of Lemma \ref{lem:smoothness-me}.
\end{proof}

\begin{proof}[Proof of Lemma \ref{lem:bounded-variance-me}]
According to Step \ref{ln:client-me-nu} of Algorithm \ref{alg:client}, let us introduce full and stochastic auxiliary cost functions $h_i(\cdot)$ and $\tilde{h}_i(\cdot)$ as follows:
\begin{align}\label{eq:persafl-me-lem-mean-1}
    h_i(\theta_i,w) &= f_i(\theta_i) + \frac{\lambda}{2}\left\lVert\theta_i-w\right\rVert^2,\\
    \tilde{h}_i(\theta_i,w,\mcD) &= \tilde{f}_i(\theta_i,\mcD) + \frac{\lambda}{2}\left\lVert\theta_i-w\right\rVert^2,\label{eq:persafl-me-lem-mean-2}
\end{align}
where due to \eqref{eq:persafl-me-lem-mean-1}, we have
\begin{align}\label{eq:persafl-me-lem-mean-3}
    \nabla\tilde{h}_i(\tilde{\theta}_i(w),w,\mcD) &= \nabla\tilde{f}_i(\tilde{\theta}_i(w),\mcD) + \lambda\left[\tilde{\theta}_i(w)-w\right],
\end{align}
hence, we can show \eqref{eq:unbiased-mean-me} as follows:
\begin{align}
    \Big\lVert\bbE_{p_i}\Big[\nabla\tilde{F}^{(c)}_i&(w,\mcD) - \nabla F^{(c)}_i(w)\Big]\Big\rVert\\
    \overset{\eqref{eq:persafl-me-full-grad},\eqref{eq:persafl-me-stoch-grad}}&{=} \left\lVert\bbE_{p_i}\big[\lambda\hat{\theta}_i(w)-\lambda\tilde{\theta}_i(w)\big]\right\rVert\label{eq:persafl-me-lem-mean-4}\\
    \overset{\eqref{eq:persafl-me-stationary},\eqref{eq:persafl-me-lem-mean-2}}&{=} \Big\lVert\bbE_{p_i}\big[\nabla f_i(\hat{\theta}_i(w))- \nabla\tilde{f}_i(\tilde{\theta}_i(w),\mcD) + \nabla\tilde{h}_i(\tilde{\theta}_i(w),w,\mcD)\big]\Big\rVert\\
    &= \Big\lVert\bbE_{p_i}\big[\nabla f_i(\hat{\theta}_i(w))- \nabla f_i(\tilde{\theta}_i(w))\big] + \bbE_{p_i}\big[\nabla\tilde{h}_i(\tilde{\theta}_i(w),w,\mcD)\big]\Big\rVert\\
    &\leq \Big\lVert\bbE_{p_i}\big[\nabla f_i(\hat{\theta}_i(w))- \nabla f_i(\tilde{\theta}_i(w))\big]\Big\rVert + \nu\\
    \overset{\eqref{eq:smoothness}}&{\leq} L\Big\lVert\bbE_{p_i}\big[\hat{\theta}_i(w)- \tilde{\theta}_i(w)\big]\Big\rVert + \nu\\
    \overset{\eqref{eq:persafl-me-lem-mean-4}}&{=} \frac{L}{\lambda}\left\lVert\bbE_{p_i}\left[\nabla\tilde{F}^{(c)}_i(w,\mcD) - \nabla F^{(c)}_i(w)\right]\right\rVert + \nu \Rightarrow \\
    \Big\lVert\bbE_{p_i}\Big[\nabla\tilde{F}^{(c)}_i&(w,\mcD) - \nabla F^{(c)}_i(w)\Big]\Big\rVert \leq \frac{\lambda}{\lambda-L}\nu.
\end{align}
You can find the proof of \eqref{eq:bounded-variance-me} in \cite{dinh2020personalized}[Appendix A.2].
\end{proof}

\begin{proof}[Proof of Lemma \ref{lem:bounded-heterogeneity-me}]
First, note that we have
\begin{align}
    \frac{1}{n}\sum\limits_{i=1}^{n}\Big\lVert\nabla F^{(c)}_i(w) - \nabla &F^{(c)}(w)\Big\rVert^2\\ \overset{\eqref{eq:persafl-me-exact-gradient-end}}&{=} \frac{1}{n}\sum\limits_{i=1}^{n}\left\lVert \lambda(w-\hat{\theta}_i(w)) - \frac{1}{n}\sum_{j=1}^n \lambda(w-\hat{\theta}_i(w))\right\rVert^2\\
    \overset{\eqref{eq:persafl-me-stationary}}&{=} \frac{1}{n^3}\sum\limits_{i=1}^{n}\left\lVert \sum\limits_{j=1}^{n} \left[\nabla f_i(\hat{\theta}_i(w)) - \nabla f_j(\hat{\theta}_j(w))\right]\right\rVert^2\\
    \overset{\eqref{eq:gen-ineq-4}}&{\leq} \frac{1}{n^2}\sum\limits_{i=1}^{n}\sum\limits_{j=1}^{n}\left\lVert \nabla f_i(\hat{\theta}_i(w)) - \nabla f_j(\hat{\theta}_j(w))\right\rVert^2.\label{eq:persafl-me-lem-hetero-1-end}
\end{align}
So, we simplify the upper bound as follows:
\begin{align}\label{eq:persafl-me-lem-hetero-2}
    \Big\lVert \nabla f_i(\hat{\theta}_i(w)) - \nabla &f_j(\hat{\theta}_j(w))\Big\rVert^2\\
    &= \Big\lVert \nabla f_i(\hat{\theta}_i(w)) - \nabla f_i(\hat{\theta}_j(w)) + \nabla f_i(\hat{\theta}_j(w)) - \nabla f(\hat{\theta}_j(w))\nonumber\\
    &\quad+ \nabla f(\hat{\theta}_j(w)) - \nabla f(\hat{\theta}_i(w)) + \nabla f(\hat{\theta}_i(w)) - \nabla f_j(\hat{\theta}_i(w))\nonumber\\
    &\quad+ \nabla f_j(\hat{\theta}_i(w)) - \nabla f_j(\hat{\theta}_j(w))\Big\rVert^2\\
    \overset{\eqref{eq:gen-ineq-1}}&{\leq} \frac{4}{3}\Big\lVert \nabla f_i(\hat{\theta}_i(w)) - \nabla f_i(\hat{\theta}_j(w)) + \nabla f(\hat{\theta}_j(w)) - \nabla f(\hat{\theta}_i(w)) \nonumber\\
    &\qquad\quad+ \nabla f_j(\hat{\theta}_i(w)) - \nabla f_j(\hat{\theta}_j(w))\Big\rVert^2\\
    &+4\Big\lVert \nabla f(\hat{\theta}_i(w)) - \nabla f_j(\hat{\theta}_i(w)) + \nabla f_i(\hat{\theta}_j(w)) - \nabla f(\hat{\theta}_j(w)) \Big\rVert^2\\
    \overset{\eqref{eq:gen-ineq-4}}&{\leq} 4\Big\lVert \nabla f_i(\hat{\theta}_i(w)) - \nabla f_i(\hat{\theta}_j(w))\Big\rVert^2\label{eq:persafl-me-lem-hetero-2-1}\\
    &+4\Big\lVert \nabla f(\hat{\theta}_j(w)) - \nabla f(\hat{\theta}_i(w))\Big\rVert^2\label{eq:persafl-me-lem-hetero-2-3}\\
    &+4\Big\lVert \nabla f_j(\hat{\theta}_i(w)) - \nabla f_j(\hat{\theta}_j(w))\Big\rVert^2\label{eq:persafl-me-lem-hetero-2-end}\\
    &+8\Big\lVert \nabla f_i(\hat{\theta}_j(w)) - \nabla f(\hat{\theta}_j(w))\Big\rVert^2\label{eq:persafl-me-lem-hetero-2-2}\\
    &+8\Big\lVert \nabla f(\hat{\theta}_i(w)) - \nabla f_j(\hat{\theta}_i(w))\Big\rVert^2\label{eq:persafl-me-lem-hetero-2-4}
\end{align}
Note that we can bound \eqref{eq:persafl-me-lem-hetero-2-2} and \eqref{eq:persafl-me-lem-hetero-2-4} according to Lemma \ref{eq:bounded-heterogeneity}:
\begin{align}\label{eq:persafl-me-lem-hetero-3}
    \frac{1}{n}\sum\limits_{i=1}^{n}\Big\lVert \nabla f_i(\hat{\theta}_j(w)) - \nabla f(\hat{\theta}_j(w))\Big\rVert^2& \overset{\eqref{eq:bounded-heterogeneity}}{\leq} \gamma_g^2,
\end{align}
and also given the fact that function $f(\cdot)$ as well as each function $f_i(\cdot)$ are $L$-smooth, we can bound \eqref{eq:persafl-me-lem-hetero-2-1}, \eqref{eq:persafl-me-lem-hetero-2-3}, and \eqref{eq:persafl-me-lem-hetero-2-end} as follows:
\begin{align}\label{eq:persafl-me-lem-hetero-4}
    \Big\lVert \nabla f_i(\hat{\theta}_i(w)) &- \nabla f_i(\hat{\theta}_j(w))\Big\rVert^2\\
    & \leq L^2\Big\lVert \hat{\theta}_i(w) - \hat{\theta}_j(w)\Big\rVert^2\\
    & = \frac{L^2}{\lambda^2}\Big\lVert \lambda\left[\hat{\theta}_i(w) - w\right] - \lambda\left[\hat{\theta}_j(w) - w\right]\Big\rVert^2\\
    \overset{\eqref{eq:persafl-me-exact-gradient-end}}&{=} \frac{L^2}{\lambda^2}\Big\lVert \nabla F^{(c)}_i(w) - \nabla F^{(c)}_j(w)\Big\rVert^2\\
    \overset{}&{=} \frac{L^2}{\lambda^2}\Big\lVert \nabla F^{(c)}_i(w) - \nabla F^{(c)}(w) + \nabla F^{(c)}(w) - \nabla F^{(c)}_j(w)\Big\rVert^2\\
    \overset{\eqref{eq:gen-ineq-4}}&{\leq} \frac{2L^2}{\lambda^2}\left[\Big\lVert \nabla F^{(c)}_i(w) - \nabla F^{(c)}(w)\Big\rVert^2 + \Big\lVert\nabla F^{(c)}(w) - \nabla F^{(c)}_j(w)\Big\rVert^2\right].\label{eq:persafl-me-lem-hetero-4-end}
\end{align}
Therefore, according to \eqref{eq:persafl-me-lem-hetero-1}-\eqref{eq:persafl-me-lem-hetero-4-end}, we have
\begin{align}\label{eq:persafl-me-lem-hetero-1}
    \frac{1}{n}\sum\limits_{i=1}^{n}\Big\lVert\nabla F^{(c)}_i(w) - \nabla F^{(c)}(w)\Big\rVert^2 &\leq 16\gamma_g^2 + \frac{48L^2}{n\lambda^2} \sum\limits_{i=1}^{n}\Big\lVert\nabla F^{(c)}_i(w) - \nabla &F^{(c)}(w)\Big\rVert^2 \Rightarrow\\
    \frac{1}{n}\sum\limits_{i=1}^{n}\Big\lVert\nabla F^{(c)}_i(w) - \nabla F^{(c)}(w)\Big\rVert^2 &\leq \frac{16\lambda^2\gamma_g^2}{\lambda^2 - 48L^2},
\end{align}
which concludes the proof.
\end{proof}
Now, we are ready to state the proof of Theorem \ref{thm:persafl-me}.
\begin{proof}[Proof of Theorem \ref{thm:persafl-me}]
We write $\tilde{\nabla}F^{(c)}_{i}\left(w\right) = \nabla \tilde{F}^{(c)}_i\left(w, \mcD_i\right)$ to simplify \eqref{eq:persafl-me-stoch-grad}. Then, the update rule for Algorithms \ref{alg:server} \& \ref{alg:client} under \textcolor{seagreen}{Option C} can be written as follows:
\begin{align}\label{eq:persafl-me-update-rule-delay}
    w^{t{+}1} = w^{t} - \eta\beta \sum\limits_{q=0}^{Q{-}1} \tilde{\nabla} F^{(c)}_{i_t}\left(w_{i_t,q}^{\Omega(t)}\right),
\end{align}
where similar to \eqref{eq:persafl-maml-main-0}-\eqref{eq:persafl-maml-main-2-end}, we can show that:
\begin{align}
    \bbE F^{(c)}\left(w^{t{+}1}\right) & \leq \bbE F^{(c)}(w^{t}) -\left[\frac{\eta\beta(2Q{-}1)}{2} - 2\eta^2L_c\beta^2Q^2\right]\bbE\left\lVert\nabla F^{(c)}(w^{t})\right\rVert^2\label{eq:persafl-me-main-1-1}\\
    &+ \frac{\eta\beta Q L_c^2\left(1{+}2\eta\beta L_c\right)}{n} \underbrace{\sum\limits_{i=1}^n\sum\limits_{q{=}0}^{Q{-}1}\bbE\left\lVert w_{i,q}^{\Omega(t)}-w^{t}\right\rVert^2}_{=: S_{c_1}}\label{eq:persafl-me-main-1-2}\\
    &+ \eta\beta Q^2\mu_c^2 + 2\eta^2L_c\beta^2Q^2\left(\sigma_c^2 +\gamma_c^2\right),\label{eq:persafl-me-main-1-3}
\end{align}
with $L_c, \mu_c, \sigma_c, \gamma_c$ as defined in Lemmas \ref{lem:smoothness-me}, \ref{lem:bounded-variance-me}, and \ref{lem:bounded-heterogeneity-me}. Thus, to show the convergence rate of our method for the cost function in \eqref{eq:persafl-me}, it would only be sufficient to provide an upper bound on $S_{c_1}$. First, note that similar to \eqref{eq:afl-s5}-\eqref{eq:afl-s5-end}, we have
\begin{align}\label{eq:persafl-me-s1}
\left\lVert w_{i,q}^{\Omega(t)}-w^{t}\right\rVert^2 & \leq\tau\left(1{+}\frac{1}{\beta^2}\right)\left[\sum\limits_{s{=}t{-}\tau}^{t{-}1}\underbrace{\left\lVert w^{s{+}1}-w^{s}\right\rVert^2}_{=: S_{c_3}}\right]+\left(1{+}\beta^2\right)\underbrace{\left\lVert w^{\Omega(t)}-w_{i,q}^{\Omega(t)}\right\rVert^2}_{=: S_{c_2}}.
\end{align}
Now, if we introduce stepsize $\eta$ such that $\eta\leq \frac{1}{4L_c(Q{+}1)}$,
similar to \eqref{eq:afl-s6}-\eqref{eq:afl-s6-end} and \eqref{eq:afl-s7}-\eqref{eq:afl-s7-end}, the following two inequalities holds for $S_{c_2}$ and $S_{c_3}$:
\begin{align}\label{eq:persafl-me-s2}
\bbE[&S_{c_2}] = \bbE\left\lVert w_{i,q}^{\Omega(t)} - w^{\Omega(t)}\right\rVert^2\\
&\leq 8Q(1{+}2Q)\eta^2\Bigg[\sigma_c^2 + \bbE\Big\lVert\nabla F^{(c)}_i\left(w^{\Omega(t)}\right) - \nabla F^{(c)}\left(w^{\Omega(t)}\right) \Big\rVert^2 + \bbE\Big\lVert\nabla F^{(c)}\left(w^{\Omega(t)}\right)\Big\rVert^2
\Bigg],\nonumber
\end{align}
\begin{align}\label{eq:persafl-me-s3}
\bbE\left[S_{c_3}\right] &= \bbE\left\lVert w^{s{+}1} - w^s\right\rVert^2 \leq 8Q(1{+}2Q)\eta^2\beta^2\Bigg[\sigma_c^2 + \gamma_c^2 + \bbE\Big\lVert\nabla F^{(c)}\left(w^{\Omega(s)}\right)\Big\rVert^2
\Bigg],
\end{align}
where by denoting $\phi=8\eta^2Q^2(1{+}2Q)(1{+}\beta^2)$, we have
\begin{align}\label{eq:persafl-me-s1-2}
\frac{1}{n\phi}\bbE[S_{c_1}] &\leq (\tau^2{+}1)\left[\sigma_c^2 + \gamma_c^2\right] + \tau \sum_{s=t{-}\tau}^{t}\sum_{u=s{-}\tau}^{s} \bbE\Big\lVert\nabla F^{(c)}\left(w^{u}\right)\Big\rVert^2.
\end{align}
Then, according to \eqref{eq:persafl-me-main-1-1}-\eqref{eq:persafl-me-s1-2}, we obtain
\begin{align}
    \bbE F^{(c)}&(w^{t{+}1}) \leq \bbE F^{(c)}(w^{t}) -\left[\frac{\eta\beta(2Q{-}1)}{2} - 2\eta^2L_c\beta^2Q^2\right]\bbE\left\lVert\nabla F^{(c)}(w^{t})\right\rVert^2\label{eq:persafl-me-main-2-1}\\
    &+ 8\eta^3\beta Q^3 L_c^2(1{+}2Q)(1{+}\beta^2)\left(1{+}2\eta\beta L_c\right)\tau \left[\sum\limits_{s{=}t{-}\tau}^{t}\sum\limits_{u{=}s{-}\tau}^{s}\bbE\left\lVert \nabla F^{(c)}(w^u)\right\rVert^2\right]\label{eq:persafl-me-main-2-2}\\
    & + 8\eta^3\beta Q^3 L_c^2 (1{+}2Q)(\tau^2{+}1)(1{+}\beta^2)\left(1{+}2\eta\beta L_c\right) \left(\sigma_c^2+\gamma_c^2\right)\\
    &+ 2\eta^2L_c\beta^2Q^2\left(\sigma_c^2 +\gamma_c^2\right)\\
    &+ \eta\beta Q^2\mu_c^2,\label{eq:persafl-me-main-2-3}
\end{align}
where by averaging the terms in \eqref{eq:persafl-me-main-2-1}-\eqref{eq:persafl-me-main-2-3}, for $t=0,1,\dots T{-}1$, and rearranging them (similar to \eqref{eq:afl-main-5}-\eqref{eq:afl-main-5-end}, we can conclude the following inequality:
\begin{align}\label{eq:persafl-me-main-3}
    &\frac{1 - 4\eta\beta Q L_c - 16 \eta^2 Q^2 L_c^2(1{+}2Q)\tau(\tau{+}1)^2(1{+}\beta^2)\left(1{+}2\eta\beta L_c\right)}{T}\sum\limits_{t=0}^{T{-}1}\bbE\left\lVert\nabla F^{(c)}(w^{t})\right\rVert^2\nonumber\\
    &\leq \frac{2\left(F^{(c)}(w^{0})-\bbE F^{(c)}(w^{T})\right)}{\eta\beta QT} + 2Q\mu_c^2\\
    & + 16\eta^2 Q^2 L_c^2 (1{+}2Q)(\tau^2{+}1)(1{+}\beta^2)\left(1{+}2\eta\beta L_c\right) \left(\sigma_c^2+\gamma_c^2\right)\\
    &+ 4\eta\beta QL_c\left(\sigma_c^2 +\gamma_c^2\right)
\end{align}
Finally, by fixing $\eta = \frac{1}{Q\sqrt{L_c T}}$, for $T\geq 288 L_c (Q{+}7) (\tau{+}1)^2$, we obtain the sublinear convergence rate in Theorem \ref{thm:persafl-me}.
\end{proof}

\section{Experiments Setting}\label{app:exp-setup}
\cu{For all algorithms, we consider $Q=10$ local updates, and select the best \mbox{$\lambda\in\{20,25,30\}$} for \mtext{ME} and \mbox{$\alpha\in\{0.002,0.005,0.01\}$} for \mtext{MAML}. Moreover, we pick \mbox{$\beta\in\{0.8,1.0,1.2\}$} and fix $\eta = 0.01$. For all experiments, we consider the exact same communication setup and repeat each experiment $2$ to $3$ times and plot the test accuracy curve over time until one of the algorithms converges. We consider $n=30$ agents for all experiments. 
Moreover, for both datasets we consider $\ell$-layer CNN networks \cite{krizhevsky2010convolutional} followed by $\ell$-fully connected layers with pooling and dropout as well as cross-entropy loss, where $\ell=2$ for MNIST and $\ell=3$ for CIFAR-10. Also, for MNIST we consider $c=5$ class of samples for each client while for CIFAR we create heterogeneity by considering $c=3$ per client.

It is worth mentioning that in the implementation of algorithms with \mtext{MAML}, we approximated the Hessian-vector products via the following first-order formulation: for some small $\delta>0$,
\begin{align}
    \nabla^2 f_i(w) u \approx \frac{\nabla f_i(w + \delta u)-\nabla f_i(w - \delta u)}{\delta}.
\end{align}
Moreover, in the bi-level optimization problem for the \mtext{ME} formulation, we applied a constant $K=10$ steps of \mtext{SGD} to obtain $\tilde{\theta}_i(w)$.

}


\end{document}